\documentclass[letterpaper, 10 pt, conference]{ieeeconf}  
\usepackage{amsmath}
\DeclareMathOperator*{\argmax}{arg\,max}
\DeclareMathOperator*{\argmin}{arg\,min}
\usepackage{amsfonts} 
\usepackage{graphicx}
\usepackage{xcolor}
\usepackage{algorithm}
\usepackage{algpseudocode}
\newtheorem{remark}{Remark}

\newtheorem{theorem}{Theorem}

\newtheorem{definition}{Definition}

\IEEEoverridecommandlockouts                              
\title{\LARGE \bf Learning Hierarchical Control For Multi-Agent Capacity-Constrained Systems}

\author{Charlott Vallon$^{1}$, Alessandro Pinto$^{2}$, Bartolomeo Stellato$^{3}$, Francesco Borrelli$^{1}$ 
\thanks{$^{1}$C. Vallon and F. Borrelli are with the Department of Mechanical Engineering,
        University of California, Berkeley,  Berkeley, CA 94720}%
\thanks{$^{2}$A. Pinto is with the NASA Jet Propulsion Laboratory, Pasadena, CA 91011}%
\thanks{$^{3}$B. Stellato is with the Department of Operations Research and Financial Engineering, Princeton University, Princeton, NJ 08544}
}

\begin{document}

\maketitle
\thispagestyle{empty}
\pagestyle{empty}

\begin{abstract}
This paper introduces a novel data-driven hierarchical control scheme for managing a fleet of nonlinear, capacity-constrained autonomous agents in an iterative environment. 
We propose a control framework consisting of a high-level dynamic task assignment and routing layer and low-level motion planning and tracking layer. 
Each layer of the control hierarchy uses a data-driven Model Predictive Control (MPC) policy, maintaining bounded computational complexity at each calculation of a new task assignment or actuation input. 
We utilize collected data to iteratively refine estimates of agent capacity usage, and update MPC policy parameters accordingly.  
Our approach leverages tools from iterative learning control to integrate learning at both levels of the hierarchy, and coordinates learning between levels in order to maintain closed-loop feasibility and performance improvement of the connected architecture. 
\end{abstract}

\section{Introduction}

This paper explores iterative, data-driven hierarchical control for task assignment and control of an autonomous fleet of capacity-constrained agents. 
Hierarchical control is integral to a wide range of industrial applications for enabling complex, multi-leveled processes by breaking down decision-making into manageable levels.
Applications include power generation and distribution \cite{YAMASHITA2020109523}, \cite{SHEN2020106079}, supply chain management \cite{app13010075}, and traffic flow coordination \cite{Xu_Wang_Wang_Jia_Lu_2021, 9899376, 9145866, HAN20201}.

A common hierarchical control framework considers a high-level controller which calculates a reference path to be tracked by a low-level controller. 
To ensure stability of the connected architecture, the higher level may also transmit to the lower level the maximum allowable deviation from this reference withstand without losing stability guarantees \cite{4434079, scattolini_2008, scattolini_article}. 
In \cite{koegel2022safe, findeisen_2018}, authors formalize this concept as a ``contract" between hierarchy levels: if the high-level planner takes certain restrictions (contracts) into account while path-planning (e.g. from reachability and invariance tools), the low-level controller can guarantee a bounded maximum deviation of the system state from the planned path. 
A variety of methods for calculating such maximum deviation bounds are proposed in \cite{10197175}, \cite{HerbertCHBFT17}, \cite{vermillion2013stable}.

Other works use reachability analysis to ensure feasibility of the hierarchical control architecture.
Authors in~\cite{9147685}, \cite{koeln_article} design a two-level hierarchical Model Predictive Controller (MPC) for a linear time-invariant system in which the low-level MPC terminal constraint is the backward reachable set of the reference path given by the the high-level MPC,
allowing the low-level controller the flexibility to improve upon the reference path.
In~\cite{vallon2020datadriven}, a low-level controller tracks a high-level reference path as closely as possible while remaining in an offline-determined safe set. 

Several hierarchical control approaches have been developed for multi-agent systems, where feasibility guarantees, when provided, are typically demonstrated using simple reachability methods for linear systems \cite{raghuraman2022hierarchical}, \cite{article}, \cite{6669555}.

Unlike methods discussed in the literature, our contribution addresses \textit{safe, data-driven} hierarchical control of a fleet of nonlinear, capacity-constrained agents operating in an iterative environment, where tasks are repeated multiple times.
A high-level centralized non-convex controller dynamically assigns a route (``task assignment") to each agent, and agents are then controlled with a low-level MPC that executes, and may improve upon, the assigned route. 
We demonstrate how to integrate learning at each level of the architecture in order to enhance control performance at each additional iteration, while maintaining guarantees of constraint satisfaction.

We show how to 
\begin{enumerate}
    \item use collected agent trajectory data to safely improve high-level task assignments, 
    \item design a low-level MPC that satisfies and can safely improve upon high-level task assignments, and
    \item prove the feasibility and iterative performance improvement of the constructed data-driven hierarchical controller.
\end{enumerate}


\section{Notation}
Unless otherwise indicated, all introduced vectors are column vectors. Parentheses $(~ )$ stack vectors vertically, and brackets $[ ~]$ stack vectors horizontally.
The notation $a = {\bf{0}}_{b}$ indicates a vector $a \in \mathbb{R}^{b}$, where $a_{i} = 0,~ i \in [1,b]$. 

\section{Problem Formulation}\label{sec:pf}

We explore the problem of dynamic task assignment and control for a set of $M$ autonomous, capacity-constrained agents working to complete a shared set of tasks.
``Dynamic task assignment" refers to assigning a subset of tasks to each agent (e.g. which tasks will be completed by which agent) as a function of the agent state. 
Examples include a group of electric delivery vans maximizing total serviced customers each day, or a fleet of battery-powered Lunar rovers aiming to perform as many exploratory or maintenance tasks as possible in a given time frame.

Dynamic task assignment can be modeled in a number of ways; here, we consider path-planning over a graph \cite{8968151, 7539623, article47}. 
We model the set of tasks as a graph $G(V, E)$, as depicted in Fig.~\ref{fig:enter-label}. Each node $v_i \in V,~ i \in \{1, 2, \dots, |V|\}$ represents a task, and each edge $e_{i,j} \in E \subseteq V \times V$ represents the possibility to complete the tasks associated with nodes $v_i$ and $v_j$ in sequence.
\begin{figure}[t]
    \centering
    \includegraphics[width = 0.65\columnwidth]{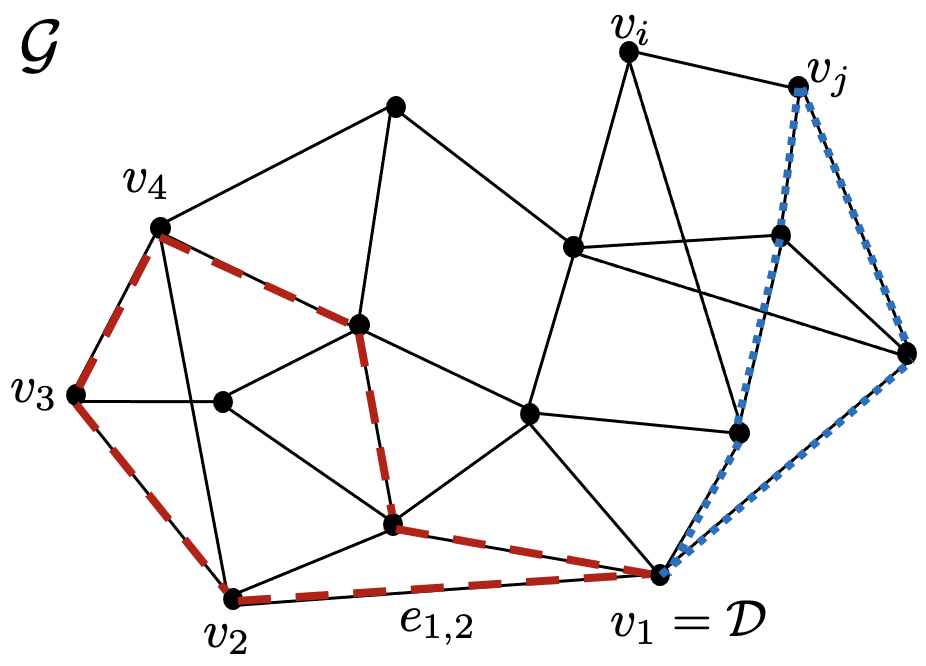}
    \caption{Graph $G$ models tasks (nodes $v_i$) and connections between tasks (edges $e_{i,j}$). A fleet of $M$ agents initially at $v_1=D$ jointly completes as many tasks as possible while without depleting capacities before returning to $v_1=D$. A task assignment for $M=2$ agents is shown in red and blue.}
    \label{fig:enter-label}
\end{figure}
The node $v_1 = D$ represents the initial state of all the agents, and where they must return after completing their assigned tasks.
If nodes represent geographical locations, this could correspond to the base station or depot. 
The set $N(i)$ returns the list of nodes which neighbor~$v_i$ (i.e. $\forall j \in N(i), ~e_{i,j} \in E$). 

Without loss of generality, and for sake of notional simplicity, agents are assumed to have the same dynamics. 
The state of the $m$th agent at time step $t$ is denoted as $x^m_t \in \mathbb{R}^n$, subject to discrete-time nonlinear dynamics
\begin{align} \label{eq:agentdyn}
    x^m_{t+1} = f(x^m_t, u^m_t),
\end{align}
where $u^m_t \in \mathbb{R}^u$ is the actuation input applied to the $m$th agent at time step $t$.  
The states and inputs are constrained as $x_t \in \mathcal{X}$ and $u_t \in \mathcal{U}$. 
A subset of each agent's state $c \in \mathbb{R}^L,~ L \leq n$ describes its $L$ capacities, which may include any limited resource, such as battery charge or travel duration:
\begin{align}
    x^m_t &= (c^m_t, \chi^m_t)\label{eq:agentstatebreakdown}\\
    c^m_{t} &= (c^m_{t,1}, c^m_{t,2}, \dots, c^m_{t,L})\label{eq:capstate},
\end{align}
where $c^m_{t,l}$ is the value of the $l$th capacity state of agent $m$ at time step $t$, and $\chi^m_t \in \mathbb{R}^{n-L}$ are the agent's non-capacity states.
The capacity and non-capacity states evolve as
\begin{align}
    c^m_{t+1} &= c^m_t + g_c(\chi^m_t, u^m_t) \label{eq:capacitydynamics} \\
    \chi^m_{t+1} &= g_{\chi}(\chi^m_t, u^m_t),\label{eq:noncapacitydynamics}
\end{align}
where $g_c(\cdot, \cdot) \geq 0$ (component-wise positive), and are subject to constraints 
\begin{align}\label{eq:agentcons}
    & c^m_{t, l} \in \mathcal{C}_l,~ l \in [1,L], ~~ \chi^m_t \in X,
\end{align}
where $\mathcal{C}_l = [0, C_l]$ and capacity state constraints are jointly denoted as $c^m_t \in \mathcal{C}$. 
The state constraints $x^m_t \in \mathcal{X}$ are satisfied if and only if both $c^m_t \in \mathcal{C}$ and $\chi^m_t \in X$.




We introduce the set $\mathcal{N}_j$ to describe that agent $m$ is at node $v_j$ at time step $t$ if 
$x^m_t \in \mathcal{N}_j$:
\begin{align}\label{eq:nodejdef}
    \mathcal{N}_j = \big\{(c, \chi) \in \mathbb{R}^n \mid  a(\chi) = j,\quad \chi = \bar{\chi}_j \big\},
\end{align}
where the function $a(\cdot)$ maps a non-capacity state $\chi$ to a node index $j$. 
For all states in $\mathcal{N}_j$, the the non-capacity states $\chi$ have fixed value $\bar{\chi}_j$, though this value can depend on $j$.

Our goal is to find a feasible input sequence for each agent (\ref{eq:agentdyn}) that optimizes a function of the agent task assignment; here we maximize the total number of unique nodes visited. We formulate the corresponding optimal control problem as
\begin{subequations}\label{eq:mainproblem}
\begin{align}
    \max_{v,x, u, T} ~ & \sum_{m \in \mathcal{M}} \sum_{i, j \in \mathcal{V}} v^{m}_{i,j} \label{eq:cost}\\
    \text{s.t.}~~ & b(v, x, u, T) \leq 0 \label{eq:mainproblemconstraints}
\end{align}
\end{subequations}
where $\mathcal{M} = \{1, 2, \dots, M \}$ and $\mathcal{V} = \{D, 2, 3, \dots, |V|\}$.
The first decision variable ${{v}} \in \{0, 1\}^{|\mathcal{V}|\times |\mathcal{V}| \times |\mathcal{M}|}$ contains the binary decision variables $v^m_{i,j}$ that describe whether agent $m$ is routed along edge $e_{i,j}$. The objective function therefore maximizes how many total nodes are visited collectively by the agents. 
The second decision variable ${x} \in \mathbb{R}^{n\times \bar{T} \times m}$ tracks all states $x^m_t$ of each $m$th agent at each time step $t \in [0, \bar{T}]$ under the inputs ${{u}} \in \mathbb{R}^{u\times (\bar{T}-1) \times m}$ until the maximum allowed time $\bar{T}$. ${{T}} \in \mathbb{R}^{|V| \times M}$ contains the time steps $T^m_j$ when agent $m$ reaches node $v_j$. If agent $m$ never reaches node $v_j$, $T^m_j = 0$. The optimal control problem (\ref{eq:mainproblem}) simultaneously searches for a path ${{v}^{\star}}$ that maximizes the number of nodes reached by agents, and a sequence of continuous actuation inputs ${{u}^{\star}}$ that result in the agents visiting the nodes before returning to the depot at node $D$. 
The full constraint set (\ref{eq:mainproblemconstraints}) is listed in the Appendix. It includes constraints ensuring that each node is visited by at most one agent and that the planned path begins and ends at the depot. Additional constraints ensure the agents' state and input constraints are satisfied. 
This problem formulation is based on the Flow Model formulation of the capacitated vehicle routing problem; for details on this formulation, we refer to \cite{routing}.

The optimal control problem (\ref{eq:mainproblem}) is a nonlinear mixed integer program, and is difficult to solve without strong restrictions on $f$ (\ref{eq:agentdyn}), $\mathcal{X}$ (\ref{eq:agentcons}), and $\mathcal{N}_j$ (\ref{eq:nodejdef}). 
In practice, problem~\eqref{eq:mainproblem} would likely not be solved as a single shot problem; the steering, braking, and acceleration of a self-driving taxi isnot be determined at the same control level or frequency of the algorithm which  decides which clients to pick up first.
We propose combining tools from iterative learning control and hierarchical control to obtain an approximate solution to (\ref{eq:mainproblem}). 
The proposed controller has two levels. A high-level controller selects a unique sequence of nodes for each agent to visit based on the current agent state (\ref{eq:agentstatebreakdown}), and a low-level controller selects continuous actuation inputs $u_t \in \mathcal{U}$ for the agent to achieve the commanded task assignment.

We consider agents executing iterative tasks, where the set of tasks does not change over iterations.
An ``iteration" will refer to each time a set of tasks is assigned to and completed by an agent (consider a fleet of vehicles delivering packages to a subset of houses in a neighborhood each week). 
Given an initial feasible (sub-optimal) task assignment and corresponding low-level trajectories that satisfy the constraints in (\ref{eq:mainproblem}), we will utilize collected data to safely increase the number of tasks completed by the agents as more data is collected during subsequent iterations. 
Control performance improvement will come from a combination of improved high-level dynamic task assignment and improved low-level trajectory planning between tasks. 
In the following section, we formalize the hierarchical architecture, and propose a method for using collected data to gain performance improvements while maintaining closed-loop feasibility.

\textit{For ease of presentation, the remainder of the paper will be written in terms of a single agent}. All described methods can be straightforwardly extended to multi-agent systems, as remarked in Sec.~\ref{ssec:multiagent}. 

\section{Iterative Data-Driven Hierarchical Control}\label{sec:iterative}
This section introduces our proposed hierarchical, iterative approach to solving (\ref{eq:mainproblem}). 
First, the two-level hierarchy is introduced. 
At each iteration, the high-level controller assigns a subset of tasks to an agent, and the low-level controller chooses actuation inputs for the agent. 
We then formalize desiderata for the resulting closed-loop system. 

\subsection{High-level controller}

The goal of the high-level controller at each iteration $r$ is to find a feasible task assignment for an agent that maximizes the total number of tasks completed and returns the agent back to the depot without depleting any capacity states. 
The high-level controller explicitly models the agent's capacity states (\ref{eq:capstate}) with an event-based agent model, updating whenever the agent reaches a new node. The model captures how much the agent's capacity states drain as a result of a particular task assignment. 

We denote the high-level state of the agent at event $k$ of iteration $r$ as $x^{H,r}_k = (n^r_k, c^r_k)$,
where $n^r_k \in \mathcal{V}$ is the index of the $k$-th node the agent visits in iteration $r$, and the vector ${c}^r_k \in \mathbb{R}^L$ represents the values of the $L$ constrained capacity states, e.g. the agent's state of charge at event $k$.

We assume the existence of an invertible mapping $g : \mathbb{R}^{n} \rightarrow \mathbb{R}^{L+1}$ that transforms the agent state (\ref{eq:agentdyn}) into a corresponding high-level state: 
\begin{align}\label{eq:gfun}
    x^H & = g(x), ~~
   g^{-1}(x^H)= \{x~|~g(x) = x^H\}.
\end{align}
Note that $g$ and its inverse only need to be well-defined for $x \in \mathcal{N}_j$, for $j \in \mathcal{V}$. 

As in (\ref{eq:agentcons}), each capacity state is constrained according to $c^r_{k,l} \in {\mathcal{C}}_l = [0, {C}_l],~ l \in [1,L]$. We compactly write these constraints as  \[c^r_k \in \mathcal{C}.\]
In the optimal control problem formulation below, we will assume without loss of generality that capacity state values are initialized at $0$ at the beginning of the task, and increase during task execution. 
This can be straightforwardly adapted to capacity states decreasing along routes (e.g. a vehicle's battery will initially be fully charged and then drain over a task assignment).

The high-level controller models the agent dynamics as 
\begin{subequations}\label{eq:hldynamics}
    \begin{align}
    n^r_{k+1} &= \sum_{ j \in \mathcal{V}} u^{H,r}_{k, n^r_k , j} j \\
    c^r_{k+1,l} &= c^r_{k,l} + \sum_{j \in \mathcal{V}} u^{H,r}_{k, n_k^r , j} \hat{\theta}^{r}_{n_k^r, j, l}~~~ l \in [1,~ L],\label{eq:11b}
\end{align}
\end{subequations}
where the input ${u}_k^{H,r} \in \{0,1\}^{|V| \times |V|}$ contains binary variables indicating a specific task assignment, with $u^{H,r}_{k,i,j} = 1$ indicating that according to the plan made at event $k$ during iteration $r$, the agent shall travel from node $v_i$ to node $v_j$ along the edge $e_{i,j}$. We compactly write (\ref{eq:hldynamics}) as 
\begin{align*}
    x^{H,r}_{k+1} = f^H({x^{H,r}_{k}}, u^{H,r}_k,\hat{\theta}^{r}).
\end{align*}

The variable $\hat{\theta}^r_{i,j,l}$ represents the high-level controller's estimate at the start of iteration $r$ of how much the capacity state $c_l$ will deplete during the agent's transition from node $v_i$ to node $v_j$. 
This corresponds to the high-level controller's approximation of the model (\ref{eq:capacitydynamics}) which describes the capacity dynamics at the low-level.
Compactly, we refer to estimates over all edges and capacity states at the start of iteration $r$ as $\hat{\theta}^r$.
The true value of ${\theta}$ is unknown to the high-level controller (and depends on the low-level controller); 
we only assume an initial bound ${\theta} \in \Theta$ is known.
At each iteration $r$, the high-level controller selects a control action based on a nominal estimate of the capacity model parameters $\hat{\theta}^r$ and its estimated worst-case bound ${\hat{\Theta}}^{r} \subseteq \Theta$. As more data is collected, the high-level controller aims to improve $\hat{\theta}^r$ and shrink ${\hat{\Theta}}^{r}$. 

\begin{remark}
The model in (\ref{eq:hldynamics}) is event-driven, updating whenever the agent reaches a new node; $c^r_k$ (\ref{eq:11b}) is the capacity state of the agent at event $k$ of task iteration $r$, when the agent reaches the $k$th visited node. In contrast, the model in (\ref{eq:agentdyn}) is time-based, updating at a fixed time interval; $c^r_t$ (\ref{eq:capstate}) is the capacity state of the agent at time $t$ of iteration $r$.
\end{remark}


\begin{remark}
Note that the high-level dynamics (\ref{eq:hldynamics}) may be compactly written as a linear event-varying system
\begin{align*}
    n^r_{k+1} &= B_n u^r_k \\
    c^r_{k+1} &= c^r_k + B_c(\hat{\theta}^{r}) u^r_k
\end{align*}
with $B_n$ and $B_c(\hat{\theta}^r)$ defined in the Appendix. 
For clarity, we use the notation in (\ref{eq:hldynamics}) throughout this paper. 
\end{remark}

At event $k$ of iteration $r$, the high-level controller's objective is to find a task assignment satisfying
\begin{subequations}\label{eq:highlevelproblem-iter}
\begin{align}
    u_k^{H,r}(x^{H,r}_k, \hat{\theta}^r, \hat{\Theta}^r) =  
    \argmax_{{u}, {c}} ~ & \sum_{i, j \in \mathcal{V}} u_{i,j} \label{eq:cost}\\
    s.t.~~ & b_2(u, c, \hat{\theta}^r, \hat{\Theta}^r) \leq 0 \label{eq:highlevelconstraints}
\end{align}
\end{subequations}
which maximizes how many nodes the agent visits while satisfying capacity constraints before returning to the depot to re-charge. The variable ${c} \in \mathbb{R}^{|V| \times |V| \times L}$ tracks the capacity states along the routes, with $c_{i,j,l}$ tracking the predicted value of the agent's $l$th capacity state as the agent departs node $v_i$ towards node $v_j$.
The full constraints (\ref{eq:highlevelconstraints}) are listed in the Appendix. (\ref{eq:highlevelproblem-iter}) primarily differs from (\ref{eq:mainproblem}) in its abstraction of the agent model (\ref{eq:hldynamics}) and by enforcing capacity constraints robustly with respect to the estimated bound set $\theta \in \hat{\Theta}^{r}$, since the true value of $\theta$ is not known.

We propose approximating (\ref{eq:highlevelproblem-iter}) with a data-driven MPC policy
\begin{align}\label{eq:ddhl}
    \pi^{H,r}(x^{H,r}_k, \hat{\theta}^r, \hat{\Theta}^r)~:~\mathcal{S}^{H,r} \rightarrow \{0,1\}^{|V|\times|V|}
\end{align}
which, at iteration $r$, maps a high-level state $x^{H,r}_k \in \mathcal{S}^{H,r}$ to a valid task assignment. 
The closed-loop high-level system under such a policy is 
\begin{subequations}\label{eq:hl-closedloop}
    \begin{align}
    u^r_k &= \pi^{H,r}(x^{H,r}_k, \hat{\theta}^r, \hat{\Theta}^r)\\
    n^{r,cl}_{k+1} &= \sum_{j \in \mathcal{V}} u^r_{k, n^r_k , j} j \\
    c^{r,cl}_{l, k+1} &= c^r_{l,k} + \sum_{j \in \mathcal{V}} u^r_{k, n^r_k , j} \theta_{n^r_k, j, l},~ l \in [1,~ L].
\end{align}
\end{subequations}

Given an initial high-level trajectory (\ref{eq:hl-closedloop}) corresponding to very conservative control performance, our aim is to first estimate $\pi^{H,1}$, $\mathcal{S}^{H, 1}$, $\hat{\theta}^{1}$, and $\hat{\Theta}^{1}$, and then use data collected during each subsequent task iteration $r$ to update $\pi^{H,r+1}$, $\mathcal{S}^{H, r+1}$, $\hat{\theta}^{r+1}$, and $\hat{\Theta}^{r+1}$ in order to improve control performance with each iteration. This is detailed in Sec.~\ref{sec:lmpc}.

\subsection{Low-level controller}

The goal of the low-level controller at each iteration $r$ is to select actuation inputs that execute the task assigned by the high-level controller. We consider such a task to consist of reaching the next node assigned by the high-level controller, denoted $n^{r\star}_{k+1|k}$. 
The low-level controller uses the time-based agent model as in (\ref{eq:agentdyn}), where the state of the agent at time step $t$ is denoted as $x_t \in \mathbb{R}^n$, subject to discrete-time dynamics
\begin{align} \label{eq:lowleveldyn}
    x_{t+1} = f(x_t, u_t),
\end{align}
where $u_t \in \mathbb{R}^u$ is the input applied at time step $t$. The low-level states and inputs are constrained as $x_t \in \mathcal{X}$ and $u_t \in \mathcal{U}$ (as in (\ref{eq:agentdyn})). 
At time step $t$ of iteration $r$ while moving from node $n^r_{k}$ to node $n^{r\star}_{k+1|k}$, the low-level controller's objective is to find an input sequence satisfying
\begin{align}\label{eq:lowleveliter}
    u^{\star}(x^{r}_t, n^r_{k}, n^{r\star}_{k+1|k}) &= \nonumber\\
    \argmin_{u, T} ~~ & \sum_{\tau = t}^{T}~ h(x_{\tau|t}, n^{r\star}_{k+1|k})\\
    \text{s.t.}~~& x_{\tau+1|t} = f(x_{\tau|t}, u_{\tau|t}) \nonumber\\
    & x_{\tau|t} \in \mathcal{X} ~~~~~~~ \forall \tau \in [t,t+T] \nonumber\\
    & u_{\tau|t} \in \mathcal{U}~~~ \forall \tau \in [t,t+T-1] \nonumber\\
    & x_{t|t} = x^{r}_t \nonumber\\
    & x_{t+T|t} \in \mathcal{N}_{n^{r\star}_{k+1|k}} \nonumber\\
    u^r_t = u^{\star}_{t|t} ~~~~~~~~~\nonumber&
\end{align}
which searches for a low-level input sequence that minimizes an objective function $h(x,j)$ over the duration of the task and drives the agent to the reference node indicated by the high-level controller. 
The objective function has the property
\begin{align}\label{eq:lowlevelh}
    h(x,j) \begin{cases}
			= 0 & x \in \mathcal{N}_j\\
            \geq 0 & \text{else}.
		 \end{cases}
\end{align}

We propose approximating the solution to (\ref{eq:lowleveliter}) with a data-driven MPC policy 
\begin{align}\label{eq:ddll}
    \pi^{L,r}(x^r_t, n^r_k, n^{r\star}_{k+1|k})~:~\mathcal{S}^{L,r} \rightarrow \mathbb{R}^u
\end{align}
which, at iteration $r$, maps an agent state in a set $\mathcal{S}^{L,r}$ to an input. 
The closed-loop system under this policy is
\begin{subequations}\label{eq:ll-closedloop}
\begin{align}
    u^{L,r}_t &=  \pi^{L,r}(x^{r,cl}_{t}, n^r_k, n^{r\star}_{k+1|k}) \\
    x^{r,cl}_{t+1} &= f(x^{r,cl}_{t}, u^{L,r}_t).
\end{align}
\end{subequations}
Given initial closed-loop low-level trajectories (\ref{eq:ll-closedloop}) corresponding to very conservative control performance, our aim is to first estimate $\pi^{L,1}$ and $\mathcal{S}^{L, 1}$, and then use additional data collected during each task iteration $r$ to update $\pi^{L,r+1}$ and $\mathcal{S}^{L, r+1}$ to improve control performance with each iteration while maintaining feasibility. This will be detailed in Sec.~\ref{sec:lmpc}.



\subsection{Hierarchical Control Architecture}
The high-level and low-level control policies (\ref{eq:ddhl}) and (\ref{eq:ddll}) are assembled into a hierarchical control architecture  as outlined in Alg.~\ref{alg:hierarchy}. 
\begin{algorithm}
\caption{Iterative Hierarchical Control}\label{alg:hierarchy}
\begin{algorithmic}
\Require $\pi^{H,1}, ~\mathcal{S}^{H,1}, ~\hat{\theta}^{1}, ~\hat{\Theta}^{1}, ~\pi^{L,1}, ~\mathcal{S}^{L,1}, ~ x^{1}_0 =x_S$
\State $r \gets 1$
\While{$\text{True}$}
\State $k \gets 0,~t \gets 0$
\State $x^{r}_t = x_S,~ x^{H,r}_k = g(x_S)$
\For{$k = 0, 1, \dots $}
\State $u^r_k = \pi^{H,r}(x^{H,r}_k, \hat{\theta}^{r}, \hat{\Theta}^{r})$ \Comment{(\ref{eq:highlevellmpc})}
\State $n^{r*} = \sum_{j \in \mathcal{V}} u^r_{k, n^r_k , j} j$
\State $i \gets n^r_k,~ j \gets n^{r*}$
\While{$x^{r}_t \not\in \mathcal{N}_{n^{r*}}$}
\State $u^{r}_t=  \pi^{L,r}(x^{r}_{t},i, j)$ \Comment{(\ref{eq:lowlevellmpc})}
\State $x^{r}_{t+1} = f(x^{r}_{t}, u^{r}_t)$
\State $t \gets t+1$
\EndWhile
\State $x_{k+1}^{H,r} = g(x^{r}_t)$
\If{$n_{k+1}^{H,r} = D$}
\State $r \gets r+1$
\State $\text{Construct } \hat{\theta}^{r}, \hat{\Theta}^{r}$ \Comment{(\ref{eq:thetaupdate}), (\ref{eq:bigthetaupdate})}
\State $\text{Construct } \mathcal{S}^{H,r},  \mathcal{S}^{L,r}$ \Comment{(\ref{eq:highlevelss-update}), (\ref{eq:lowlevelss-update})}
\State $\text{Construct } \pi^{H,r}, \pi^{L,r} $ \Comment{(\ref{eq:highlevellmpc}), (\ref{eq:lowlevellmpc})}
\State $\text{Break. Iteration $r$ complete.}$
\EndIf
\EndFor
\EndWhile
\end{algorithmic}
\end{algorithm}
At the start of each iteration $r \geq 1$, the agent is located at node $v_1 = D$ (the depot in our example), with the same initial state $x^{r}_0 = x_S \in \mathcal{N}_D$. 
The high-level controller (\ref{eq:ddhl}) assigns the next node for the agent to travel to, and the low-level controller (\ref{eq:ddll}) selects the corresponding low-level actuation inputs. Once the agent has reached the assigned node, the high-level controller assigns a new node. This process repeats until the agent has returned to the depot, at which point the policies $\pi^H$ and $\pi^L$ and their associated parameters $\mathcal{S}^H$, $\mathcal{S}^L$, $\hat{\theta}$, and $\hat{\Theta}$ are updated using closed-loop trajectory data from iteration $r$.

After iteration $r$, we write the agent's resulting closed-loop trajectory produced by Alg.~\ref{alg:hierarchy} as
\begin{align}\label{eq:arch-closedloop}
    &{\bf{x}^r}(x_S, \pi^{H,r}, \mathcal{S}^{H,r}, \hat{\theta}^r,\hat{\Theta}^r, \pi^{L,r},\mathcal{S}^{L,r}) = \\
    & ~~~~=[x_S, x^r_1, \dots, x^r_{\tau 1}, x^r_{\tau 1 + 1}, \dots,  x^r_{\tau 2}, \dots, x^r_{\tau K}] \nonumber \\ 
    & x^r_{0} \in \mathcal{N}_D \nonumber \\
& x^r_{\tau i} \in \mathcal{N}_j~~~~~~~~~~~~~~~~~~~~~~~~~~~~~~~ \forall i \in [1,K],~\exists j \in \mathcal{V} \nonumber \\
    &x^r_{\tau K} \in \mathcal{N}_D \nonumber \\
    &x^r_{t+1} = f(x^r_t, \pi^{L,r}(x^r_t))  ~~~~~~~~~~~~~~~~~~~ \forall t \in [0,\tau K-1] \nonumber \\
    & \pi^{L,r}(x^r_t) \in \mathcal{U} ~~~~~~~~~~~~~~~~~~~~~~~~~~~~~~~\forall t \in [0,\tau K-1] \nonumber \\
    &x^r_t \in \mathcal{X} ~~~~~~~~~~~~~~~~~~~~~~~~~~~~~~~~~~~~~~~~~~~ \forall t \in [0,\tau K]\nonumber \\
    &g(x^r_{\tau (i+1)}) = f^H(x^r_{\tau i},\pi^{H,r}(g(x^r_{\tau i}), \hat{\theta}^r, \hat{\Theta}^r), \theta) ~ i \in [1,K]\nonumber 
\end{align}
Throughout the rest of the paper, we will succinctly refer to (\ref{eq:arch-closedloop}) as ${\bf{x}}^r$, understanding that at iteration $r$ the control parameters are $\pi^{H,r}, \mathcal{S}^{H,r}, \hat{\theta}^r, \hat{\Theta}^r, \pi^{L,r},$ and $\mathcal{S}^{L,r}$. 

\subsection{Properties}
A significant challenge in hierarchical control is the coordination of control levels operating at different frequencies and using different agent models. 
Without coordination, the high-level controller may assign task routes that the agents cannot achieve while satisfying low-level state and actuation constraints. 
We define desirable properties for a control architecture as in Alg.~\ref{alg:hierarchy}, which can be achieved with proper coordination.

\begin{definition}[Iterative Feasibility]
    A control architecture as in Alg.~\ref{alg:hierarchy} is \textit{iteratively feasible} from an initial agent state $x_S$ if the feasibility of ${\bf{x}}^{r}$ implies the feasibility of ${\bf{x}}^{r+1}$ for all $r \geq 1$.   
\end{definition}
Iterative feasibility ensures that if the hierarchical controller ever finds a feasible task assignment and corresponding agent trajectory (or is provided with one), it will remain feasible at subsequent iterations.

We define the function
    \begin{align*}
         g_{\tau}(x_0, x_1, \dots, x_T) = [x^{H}_0, x^{H}_1, \dots, x^{H}_K] = {\bf{x}}^H
    \end{align*}
mapping an agent trajectory to the corresponding high-level state trajectory. 
Given a high-level state trajectory, we define 
\begin{align}\label{eq:hierarchicalcost}
    V_n^H([x^{H}_0, x^{H}_1, \dots, x^{H}_K])
\end{align}
to count the number of unique nodes apart from the depot visited in the trajectory.
We can now define the function 
\begin{align*}
    V_n({\bf{x}}^r) = V_n^H(g_{\tau}({\bf{x}}^r))
\end{align*}
which maps a closed-loop agent trajectory to how many unique nodes the agent visited (tasks the agent accomplished) before returning to the depot.
This is the function the hierarchical controller in Alg.~\ref{alg:hierarchy} aims to maximize at each iteration.

\begin{definition}[Iterative Performance Improvement]
    A control architecture as in Alg.~\ref{alg:hierarchy} exhibits \textit{iterative performance improvement} from an initial state $x_S$ if  $V_n({\bf{x}}^r) \leq V_n({\bf{x}}^{r+1})$. 
\end{definition}
Iterative Performance Improvement ensures that closed-loop performance will not suffer as more task iterations are performed and more data is collected. 
 
Whether a hierarchical control architecture satisfies the properties defined above depends on the methodology for using collected closed-loop data to update the policy parameters. 
It is not hard to prove that in the case of no model uncertainty,  
the properties are satisfied by using a low-level controller that reduces capacity consumption on a particular task with each subsequent iteration (e.g. capacity consumption on the $r$th traversal of edge $e_{i,j}$ is no greater than during traversal $(r+1)$ of edge $e_{i,j}$). In this case, a high-level planner could safely utilize the most recent capacity consumption along each edge as a robust estimate of future capacity consumption. 

There are two significant challenges. First, finding a computationally efficient low-level controller that guarantees capacity consumption reduction is not trivial. 
Second, depending on the size of graph $G$, dynamically calculating a new high-level task assignment can become computationally intractable. 
In the following section we describe our proposed formulation using data-driven MPC for achieving iterative feasibility and iterative performance improvement in a computationally efficient way. 
We show how to design a low-level controller with fixed, short horizon that achieves capacity consumption reduction with each iteration, and show how to design a high-level controller with a fixed, short horizon that safely and dynamically allocates tasks to agents.


 
\section{Hierarchical LMPC}\label{sec:lmpc}

We now introduce our proposed data-driven policy formulations and parameter update procedures. 
Each level of the hierarchical controller utilizes a data-driven MPC policy.
We adapt ideas from Learning Model Predictive Control (LMPC), an approach for iteratively improving on an initial provided feasible trajectory \cite{lmpc}. 
The technique relies on the realization that any state in a given feasible trajectory is a subset of the maximum controllable set to the task goal. 
Closed-loop data collected at each iteration is used to successively improve the construction of an MPC terminal set and terminal cost, resulting in iterative closed-loop performance improvement while maintaining feasibility guarantees.
Here we demonstrate how to adapt this idea for safe data-driven hierarchical control. 

As defined in (\ref{eq:arch-closedloop}), we denote with ${\bf{x}}^r$ the agent closed-loop trajectory during iteration $r$, and with 
${\bf{x}}^{H,r} = g_{\tau}({\bf{x}}^r)$ the corresponding sequence of high-level states.
At the beginning of each iteration $r$, the trajectories collected in the previous iterations are used to formulate data-driven policies $\pi^{H,r}$ and $\pi^{L,r}$ and their parameters $\mathcal{S}^{L,r}$, $\mathcal{S}^{H,r}$, $\hat{\theta}^r$, and $\hat{\Theta}^r$.

\subsection{Parameter Calculation}
We introduce two operators. The indexing operator $\mathcal{I}({\bf{x}}^r,i,j)$ returns the time range during which the agent travels from node $v_i$ to node $v_j$ in the trajectory ${\bf{x}}^r$. $\mathcal{I}({\bf{x}}^r,i,j,-1)$ returns the time step at which the agent reaches node $v_j$ after departing node $v_i$. If the agent does not travel from node $v_i$ to node $v_j$ in the trajectory ${\bf{x}}^r$, $\mathcal{I}({\bf{x}}^r,i,j)$ is empty. 

The capacity change operator $\Omega_l({\bf{x}}^r,i,j)$ calculates how much the agent's $l$th capacity state $c_l$ depleted while traveling from node $v_i$ to node $v_j$ in the trajectory ${\bf{x}}^r$. 

Before beginning iteration $r$, the capacity depletion estimate $\hat{\theta}^{r}$ is calculated as
\begin{align}\label{eq:thetaupdate}
    \hat{\theta}^{r}_{i,j,l} = \min_{p \in [0,r-1]} \Omega_l({\bf{x}}^{p},i,j),
\end{align}
where each $\hat{\theta}_{i,j,l}^{r} \in \hat{\theta}_{i,j}^{r}$ describes the smallest amount the $l$th capacity state depleted along edge $e_{i,j}$ in any previous recorded iteration. 
The bound estimate $\hat{\Theta}^{r}$ is updated as
\begin{align}\label{eq:bigthetaupdate}
    \hat{\Theta}^{r}_{i,j} = \{\theta \in \mathbb{R}^L ~|~ 0 \leq \theta_{l} \leq  \hat{\theta}_{i,j,l}^{r},~ l \in [1,L]\}.
\end{align}

At the start of iteration $r$, we define a high-level safe set
\begin{align} \label{eq:highlevelss-update}
    \mathcal{S}^{H,r} & = \bigcup_{p=0}^{r-1} s({\bf{x}}^{H,p}, \hat{\theta}^{p+1})
\end{align}
where the function $s$ is defined by Alg.~\ref{alg:cap}.
\begin{algorithm}
\caption{$s : {\bf{x}}^H, \hat{\theta} \rightarrow \tilde{\bf{x}}$}\label{alg:cap}
\begin{algorithmic}
\State ${\bf{n}} \gets \text{nodes in } {\bf{x}}^H,~~K \gets \text{length}({\bf{n}})$
\State $\bar{c}_{l,K} \gets {C}_l, ~~\bar{c}_{l,0} \gets 0, ~~~~\forall l \in [1,L]$
\For{$k = K-1:-1:1$}
\State $\bar{c}_{l,k} \gets \bar{c}_{l,k+1} + \hat{\theta}_{n_k, n_{k+1},l}~~~~\forall l \in [1,L]$
\State $\tilde{x}_k = (n_k, \bar{c}_k) $
\EndFor
\State $\tilde{x}_0 = (n_0, \bar{c}_0),~  \tilde{x}_K = (n_K, \bar{c}_K)$\\
\Return $\tilde{\bf{x}}$
\end{algorithmic}
\end{algorithm}
Given a node trajectory and capacity depletion estimate $\hat{\theta}$, Alg.~\ref{alg:cap} calculates the maximum capacity state values $\bar{c}_k$ the agent could have had at each node $n_k$ along the trajectory while still satisfying capacity state constraints according to depletion rates predicted by $\hat{\theta}$. 
Thus, the set $\mathcal{S}^{H,r}$ contains the sequences of nodes traversed in all previous iterations, along with corresponding worst-case capacity state trajectories that would have been feasible for the node sequence under the high-level controller's capacity depletion estimate $\hat{\theta}$ at that iteration. 
Note this implies that if $\hat{\theta}^r \leq \hat{\theta}^{r-1}$, then each state $x^H \in \mathcal{S}^{H,r}$ represents a node and capacity state from which there is guaranteed to exist a feasible path back to the depot according to $\hat{\theta}^r$. 
The size of $\mathcal{S}^{H,r}$ grows with each additional iteration $r$.

We define a value function $Q^{H,r}(\cdot)$ over $\mathcal{S}^{H,r}$, where 
\begin{align}\label{eq:highcosttogo}
    Q^{H,r}(x^H) = \begin{cases}
        \max_{(p,k) \in W^{r}(x^H)} V^H_n({\bf{x}}^{H,p}_{k:}) & x^H \in \mathcal{S}^{H,r} \\
        -\infty & \text{else}
    \end{cases}
\end{align}
where
 \begin{align*}
    W^{r}(x^H) = &\{(p,k)~:~ p \in [0,r-1],~ k \geq 0  \nonumber\\
    & ~~~~\text{ with } x^{H,p}_k = x; ~\text{for } x^{H,p}_k \in \mathcal{S}^{H,r} \}
\end{align*}   
which corresponds a state $x^H \in {\bf{x}}^{H,p}$ with the number of nodes the agent visited after being in state $x^H$ in iteration $p$.

Low-level safe sets ${\mathcal{S}}_{i,j}^{L,r}$ are stored for each edge $e_{i,j}$ in $G$, where
\begin{align}\label{eq:lowlevelss-update}
    \mathcal{S}_{i,j}^{L,r} &= x^{p\star}_{\mathcal{I}({\bf{x}}^{p\star}, i,j)} - (c^{p\star}_{\mathcal{I}({\bf{x}}^{p\star}, i,j,0)}, {\bf{0}}_{n-L}) \\
    p{\star}& = \argmax_{p \in [0,r-1]} ~~p \nonumber\\
    & ~~~~\text{s.t.}~~ \mathcal{I}({\bf{x}}^{p}, i,j) \text{ not empty}. \nonumber
\end{align}
At the beginning of iteration $r$, the safe set $\mathcal{S}_{i,j}^{L,r}$ contains the most recently recorded trajectory between node $v_i$ and node $v_j$. This need not correspond to iteration $r-1$, as only edges assigned by the high-level controller are traveled during an iteration. Note that for each state in $\mathcal{S}_{i,j}^{L,r}$, the stored capacity values have been shifted to begin at $0$ at the start of the edge trajectory. Non-capacity states are unaltered.


We refer to the low-level stage cost $h(x,j)$ introduced in (\ref{eq:lowleveliter}) and define a low-level value function $Q^{L,r}_{i,j}(\cdot)$: 
\begin{align}
    Q^{L,r}_{i,j}(x) &= \begin{cases}
         \sum_{s=t}^{\mathcal{I}({\bf{x}}^{p\star},i,j,-1)} h(x^{p\star}_s,j) & x \in \mathcal{S}^{L,r}_{i,j} \\
        + \infty & \text{else} 
    \end{cases}\nonumber \\
    p{\star}& = \argmax_{p \in [0,r-1]} ~~p \nonumber\\
    & ~~~~\text{s.t.}~~ \mathcal{I}({\bf{x}}^{p}, i,j) \text{ not empty}. \nonumber
\end{align}
The function $Q^{L,r}_{i,j}(\cdot)$ assigns each state in $\mathcal{S}^{L,r}_{i,j}$ the minimum cost to reach the node $v_j$ along the trajectory stored in $\mathcal{S}^{L,r}_{i,j}$.

\subsection{Policy Formulation}\label{ssec:lmpcpolicy}
At each iteration $r\geq1$, the data-driven policies $\pi^{H,r}$ and $\pi^{L,r}$ are formulated using $\mathcal{S}^{L,r}$, $\mathcal{S}^{H,r}$, $\hat{\theta}^r$, and $\hat{\Theta}^r$ as follows.

\subsubsection{High-level policy $\pi^{H,r}$}
At event $k$ of the $r$th iteration, the high-level controller searches for the path of $N$ edges beginning at the current state $x^{H,r}_k = (n^{r}_{k}, c^r_k)$ that maximizes the number of unique visited nodes.
We introduce the optimal control problem by first listing the necessary constraints.

The binary variable $u \in \{0,1\}^{|V| \times |V|}$ describes how the agent will be routed, with $u_{i,j}=1$ indicating that the agent shall travel from node $v_i$ to node $v_j$.
The constraints 
\begin{subequations}\label{eq:lowlevellmpcconstraintA}
\begin{align}
& \sum_{j\in \mathcal{V}} u_{n^{r}_{k},j} = 1 \label{eq:firsthighlevelconstraint}\\
& c_{n^{r}_{k},j,l} = c^{r}_{l,k} ~~~~~~~~~~~~~~ \forall j \in \mathcal{V}, ~l \in [1,L] \label{eq:secondhighlevelconstraint} \\
& u_{i,j} = 0 ~~~~~~~~~~~~~~~~~~~~ \forall i \in \mathcal{V}, ~j \notin N(i) \\
& c_{i,j,l} = 0 ~~~~~~ \forall i \in \mathcal{V}, ~j \notin N(i), ~l \in [1,L]
\end{align}
\end{subequations}
ensure the path starts at node $n^{r}_{k}$, and initialize the values of the capacity state accumulation variables $c_{i,j,l}$ to the current capacity states at event $k$. Each variable $c_{i,j,l}$ tracks the value of the $l$th capacity state $c_l$ as the agent departs node $v_i$ towards node $v_j$. Paths can only be planned between neighboring nodes. 
The constraints
\begin{subequations}
    \begin{align}
    & \sum_{i \in \mathcal{V}} u_{i,j} \leq 1 ~~~~~~ \forall j \in \mathcal{V}\setminus D\\
    & \sum_{j \in \mathcal{V}} u_{i,j} \leq 1 ~~~~~~ \forall i \in \mathcal{V}\setminus D \\
    & \sum_{i \in \mathcal{V}} u_{i,j} \geq \sum_{i \in \mathcal{V}} u_{j,i} ~~~~ \forall j \in \mathcal{V} \label{eq:nodesharing}
\end{align}
\end{subequations}
ensure that a single path is formed, and that no paths begin from any node other than $n^{r}_{k}$. 
The constraints
\begin{subequations}
    \begin{align}
    & \sum_{i \in \mathcal{V}} c_{i,j,l} + \hat{\theta}_{i,j,l}^{r} = \sum_{i \in \mathcal{V}} c_{j,i,l} ~ \forall j \in \mathcal{V} \setminus D, ~ l = [1,L] \label{eq:37}\\
    & c_{i,j,l} \geq \sum_{j \in \mathcal{V}} \hat{\theta}^r_{j,i,l} u_{j,i} ~~ \forall i \in \mathcal{V} \setminus D, ~j \in \mathcal{V}, ~l \in [1,L]\\
    & c_{i,j,l} \leq ({C}_l - \hat{\theta}_{i,j,l}^{r}) u_{i,j}~~~ \forall i \in \mathcal{V}, ~j \in \mathcal{V},~ l \in [1,L] \label{eq:39}
\end{align}
\end{subequations}
track the evolution of the capacity state accumulation variables along the planned route, and ensure that they satisfy all capacity state constraints $\mathcal{C}$ (\ref{eq:capacitydynamics}). 
The constraints 
\begin{subequations}
\begin{align}
    & \sum_{i,j \in \mathcal{V}} u_{i,j} \leq N \\
    & \sum_{j \in {\bf{n}}^r_{:k} \setminus D} u_{i,j} = 0~~~~~~ \forall i \in \mathcal{V}\label{eq:prevpathconstraint}
\end{align}    
\end{subequations}
ensure that the path is at most length $N$ does not intersect with any nodes traversed thus far during this iteration (stored in the set ${\bf{n}}^r_{:k}$). 
Lastly, the constraints
\begin{subequations}\label{eq:lowlevellmpcconstraintF}
\begin{align}
    & \sum_{n \in {\mathcal{S}}^{H, r}} \bigg( \sum_{i\in \mathcal{V}} u_{i,n} - \sum_{i\in\mathcal{V}} u_{n,i}\bigg) = 1 \label{eq:termset}\\
    & \sum_{n \in {\mathcal{S}}^{H, r}} \bigg( \sum_{i\in \mathcal{V}} u_{i,n} - \sum_{i\in\mathcal{V}} u_{n,i}\bigg)j = \bar{n} \label{eq:laststate}\\
    & \sum_{i \in \mathcal{V}} c_{\bar{n},i,l} = \bar{c}_l ~~~~~~~~~~~~~~~~~~~~~~~~~~~ \forall l \in [1,L]\label{eq:lastc}\\
    & \bar{c} \leq \tilde{c} \label{eq:cinequality}\\
    & (\bar{n}, \tilde{c}) \in \mathcal{S}^{H, r} \label{eq:lmpcinequality}\\
    & \sum_{i \in \mathcal{V}} u_{i,j} = 0 ~~~~~~~~~~~~~~~~~~~~~~~~~~~ \forall j \in P(\bar{n}, \tilde{c}) \label{eq:lastlmpcconstraint}
\end{align}    
\end{subequations}
impose a terminal constraint on the planned path. Constraint (\ref{eq:termset}), along with (\ref{eq:nodesharing}), ensures that the planned path ends in a node in ${\mathcal{S}}^{H, r}$, a node the agent has visited in a previous task iteration. This node is labeled as $\bar{n}$ in (\ref{eq:laststate}); the corresponding vector of planned capacity states at node $\bar{n}$ is labeled as $\bar{c}$ in (\ref{eq:lastc}). 
Constraints (\ref{eq:cinequality}) and (\ref{eq:lmpcinequality}) enforce that the capacity state $\bar{c}$ planned at node $\bar{n}$ is less than the capacity state $\tilde{c}$ associated with $\bar{n}$ in the safe set (recall the formulation choice of capacity states being initialized at $0$ and increasing throughout the task).
Finally, constraint (\ref{eq:lastlmpcconstraint}) ensures that the planned path does not include any nodes visited after the terminal node in a previous iteration, a set denoted $P(\bar{n}, \tilde{c})$.

We can now formulate the optimal control problem. At event $k$ of iteration $r$, the high-level controller solves 
\begin{align}
    u^{H,r}_k & = \pi^{H,r}(x_k^{H,r}, \hat{\theta}^r, \hat{\Theta}^r) \label{eq:highlevellmpc}\\
    & = \argmax_{u, c} \sum_{i, j \in \mathcal{V}} u_{i,j} + Q^{H,r}((\bar{n}, \tilde{c})) \nonumber \\
    & ~~~~~~~~~~~ \text{s.t.} ~~ (\ref{eq:lowlevellmpcconstraintA}) - (\ref{eq:lowlevellmpcconstraintF}). \nonumber
\end{align}
The objective function maximizes the sum of two terms: the total number of unique nodes visited during the planned $N$-step path and how many additional nodes the agent visited beyond the planned terminal node in previous iterations. 
The first node in the planned $N$-step path is used as a task assignment for the low-level controller. 
Specifically, the next node assignment passed to the low-level controller is
\begin{align}
    n^{r,\star}_{k+1|k} = \sum_{j \in \mathcal{V}} u^{H,r}_{k, n_k^r, j} j.
\end{align}
When the agent reaches the reference node, (\ref{eq:highlevellmpc}) is solved again. This procedure is repeated until the terminal constraint (\ref{eq:lowlevellmpcconstraintF}) eventually forces the agent back to the depot.  

This formulation (\ref{eq:highlevellmpc}) adapts the LMPC approach introduced in \cite{lmpc} to the discrete task assignment problem. 
Note that because (\ref{eq:highlevellmpc}) is solved each time the agent reaches a new node, it is possible a more productive route can be found at event $k+1$ than at event $k$. 
If not, by enforcing the terminal constraint (\ref{eq:lowlevellmpcconstraintF}) on the planned path, the controller ensures that the agent will still be able to return to the depot from the last planned node without depleting capacity states so long as $\theta \leq \hat{\theta}^r \leq \hat{\theta}^{r-1}$. Such a path exists and is stored in $\mathcal{S}^{H,r}$.
Note that the satisfaction of the condition $\theta \leq \hat{\theta}^r \leq \hat{\theta}^{r-1}$ depends on the parameter update procedure for $\hat{\theta}^r$ and the formulation of the low-level controller; Sec.~\ref{sssec:lowlevel} proposes a formulation that accomplishes this.

\begin{remark}
As written here, problem (\ref{eq:highlevellmpc}) is solved over the full graph $G$. An alternate method is to solve (\ref{eq:highlevellmpc}) over a local graph $G'(n^r_k, N) \subseteq G$ containing only nodes within a radius $N$ of node $n^r_k$. This local graph would be updated dynamically at each event $k$.
\end{remark}


\subsubsection{Low-level policy $\pi^{L,r}$}\label{sssec:lowlevel}
At time step $t$ while moving from node $n^r_k = i$ to node $n^{r,\star}_{k+1|k}= j$ during iteration $r$, the actuation input $u^r_t$ to apply at state $x^r_t$ is calculated as
\begin{subequations}\label{eq:lowlevellmpc}
\begin{align}
u^{\star} = \argmin_{u}~~  & \sum_{k=t}^{N-1} h(x_{k|t}, u_{k|t}) + Q^{L,r}_{i,j}((\bar{c}, \bar{\chi})) \\
    s.t.~~ & x_{k+1|t} = f(x_{k|t},~u_{k|t}) \label{eq:lowleveldynamics} \\
    & x_{k|t} \in \mathcal{X} ~~~~~~~~~~~~ \forall k \in [t, t+N]\label{eq:lowlevelstate}\\
    &u_{k|t} \in \mathcal{U}, ~~~~~~ \forall k \in [t, t+N-1] \label{eq:lowlevelinput}\\
    & x_{t|t} = x^r_t \\
    & c_{t+N|t} - c_k^{r} \leq \bar{c}
    \label{eq:lowlevelinequalityterm}\\
    & \chi_{t+N|t} = \bar{\chi} \label{eq:lowlevelchi}\\
    & (\bar{c}, \bar{\chi}) \in \mathcal{S}^{L, r}_{i,j}
    \label{eq:lowlevelterm}\\
    u^r_t = \pi^{L, r}(x^r_t,i&,j) = u^{\star}_{t|t}. 
\end{align}
\end{subequations}
The controller searches at each time step $t$ for a feasible $N$-step trajectory from the current state $x^r_t$ to a state $x_{t+N|t} = (c_{t+N|t}, \chi_{t+N|t})$, where $x_{t+N|t}$ relates to a state $x_{t\star} = (\bar{c}, \bar{\chi})$ in the safe set $\mathcal{S}^{L,r}_{i,j}$ in two ways:
\begin{enumerate}
    \item the non-capacity states $\chi_{t+N|t}$ are the same as those of state $x_{t\star}$ (\ref{eq:lowlevelchi}), and
    \item the capacity states $c_{t+N|t}$ have depleted less since the beginning of the edge trajectory (where capacity states had value $c^r_k$) than those of $x_{t\star}\in \mathcal{S}^{L,r}_{i,j}$ (\ref{eq:lowlevelinequalityterm}).
\end{enumerate}
The objective function minimizes the sum of two terms: the stage cost (\ref{eq:lowlevelh}) accumulated over the planned $N$-step trajectory, and the remaining stage cost accumulated on the trajectory from $x_{t\star}$ back to the depot in a previous iteration. At time step $t$, the first planned input is applied and a new optimal control problem (\ref{eq:lowlevellmpc}) is solved at time step $t+1$.

It follows from \cite{lmpc} and construction of (\ref{eq:lowlevelh}) and (\ref{eq:lowlevellmpc}) that there exists a time step $t^{\star}$ while the low-level policy (\ref{eq:lowlevellmpc}) is being applied such that $x_{t^{\star}+N|t^{\star}} \in \mathcal{N}_j$, i.e. the planned $N$-step path reaches the goal node $v_j$.
For all time steps beyond $t^{\star}$, until the goal node is reached, the low-level controller will apply the corresponding input as planned at $t^{\star}$: 
\begin{align}\label{eq:lowlevellmpc-rec}
    &\forall t \in [t^{\star}, t^{\star}+N-1]:~ u_t^r = \pi^{L,r}(x_t^r,i,j) = u^{\star}_{t|t^{\star}}
\end{align}

Note that the terminal constraints (\ref{eq:lowlevelterm}) and (\ref{eq:lowlevelinequalityterm}) force the planned $N$-step trajectory to end in a state $x_{t+N|t}$ with non-capacity state values from which, in a previous iteration, the agent was able to reach the reference node $v_j$ despite having more depleted capacity states at that point. 
This ensures that \textit{a)} the planned trajectory ends in a state from which the agent will be able to reach node $v_j$ while satisfying state and input constraints, and \textit{b)} the agent will reach the next task node with capacity value levels required to guarantee existence of a feasible path back to the depot. 
This guarantee follows from the agent having achieved it in a previously recorded trajectory (stored in $\mathcal{S}^{L,r}_{i,j}$).

\begin{remark}
   The low-level MPC horizon and high-level MPC horizon need not be the same length $N$.
\end{remark}

\subsection{Hierarchical LMPC Architecture}\label{ssec:initialdataset}
The high-level and and low-level LMPC policies (\ref{eq:highlevellmpc}) and (\ref{eq:lowlevellmpc}) are assembled into a hierarchical control architecture as described in Alg.~\ref{alg:hierarchy}. 
Before beginning iteration $r=1$, we assume the availability of a dataset containing:
\begin{enumerate}
    \item First, an initial, feasible agent trajectory ${\bf{x}}^0$ beginning and ending at the depot. This trajectory is sub-optimal with respect to (\ref{eq:mainproblem}); it may have been possible for the agent to complete more tasks while satisfying capacity constraints. The trajectory ${\bf{x}}^0$ is used to initialize the high-level safe set $\mathcal{S}^{H,1}$ as well as the initial parameter estimates $\hat{\theta}^1$ and $\hat{\Theta}^1$.
    \item Second, a set of feasible agent trajectories along each edge $e_{i,j}$ in $G$, beginning in $\mathcal{N}_i$ and ending in $\mathcal{N}_j$. Each of these trajectories satisfies the constraints in (\ref{eq:lowleveliter}), but they may be sub-optimal with respect to (\ref{eq:lowleveliter}). These trajectories are used to initialize the low-level safe sets  $\mathcal{S}^{L,1}_{i,j}$ for each edge.
\end{enumerate}
These initial datasets may contain feasible trajectories corresponding to very conservative controllers. 

As more data is collected with further iterations, the terminal constraints and terminal costs used in $\pi^H$ and $\pi^L$ are updated to reflect improved estimates of the agent's capacity usage during task execution. 
Control performance improvement of the hierarchical control architecture comes from two separate but coordinated phenomena.
First, low-level performance improvements result in new edge trajectories that minimize control cost with respect to (\ref{eq:lowlevelh}) while reducing capacity consumption. When trajectories are discovered along an edge that reduce capacity consumption, the model estimate $\hat{\theta}$ improves, and the high-level model uncertainty bound shrinks. 
Second, the high-level controller utilizes these improved bounds to plan more productive high-level routes on the graph $G$. 
As a result, control performance is improved both separately at the low-level and high-level, and jointly with respect to (\ref{eq:hierarchicalcost}). 

The proposed architecture represents a novel, computationally efficient approach for combining tools from hierarchical control and data-driven iterative learning control to dynamically assign tasks and plan routes in a way that ensures constraint satisfaction despite using short-horizon predictions. 
Next, we prove that the high-level and and low-level LMPC policies (\ref{eq:highlevellmpc}) and (\ref{eq:lowlevellmpc}) assembled into a hierarchical control architecture as described in Alg.~\ref{alg:hierarchy} and updated according to (\ref{eq:thetaupdate}), (\ref{eq:highlevelss-update}), (\ref{eq:lowlevelss-update}) satisfy the properties of iterative feasibility and iterative performance improvement.


\subsection{Multi-Agent Formulation}\label{ssec:multiagent}
For the sake of notational simplicity, the hierarchical LMPC architecture has been presented for a single agent. However, the method can be straightforwardly applied to fleets of $M$ agents with slight modifications to the formulation of the high-level LMPC (\ref{eq:highlevellmpc}). 
In the multi-agent case, a centralized high-level LMPC (\ref{eq:highlevellmpc}) plans $M$ routes, each of length $N$, for $M$ agents. Then, $M$ decentralized LMPCs in the form of (\ref{eq:lowlevellmpc}) will guide each agent to their assigned reference nodes. 

The multi-agent high-level state of the system will concatenate the $M$ nodes and capacity states of the agents. Therefore, each point in the high-level safe set $\mathcal{S}^{H,r}$ will consist of $M$ nodes and capacity state values the agents have simultaneously held at a previous iteration. 

The formulation in (\ref{eq:highlevellmpc}) must be adapted accordingly to plan $M$ satisfactory paths. 
Specifically, constraints (\ref{eq:firsthighlevelconstraint}), (\ref{eq:secondhighlevelconstraint}) are changed to plan $M$ routes, each beginning at the current node and capacity level of an agent $m \in M$.
Additionally, the terminal constraints (\ref{eq:lowlevellmpcconstraintF}) must be changed to enforce that the $M$ planned paths jointly end in nodes and capacity state values in the higher-dimension safe set $\mathcal{S}^{H,r}$.

All properties proven for Alg.~\ref{alg:hierarchy} in Sec.~\ref{ssec:properties} hold in the multi-agent case. 

\begin{remark}
    If all $M$ agents have the same dynamics (\ref{eq:agentdyn}), the low-level safe sets $\mathcal{S}^{L}$ and capacity depletion estimate $\hat{\theta}$ can be shared amongst agents. If agents have different dynamics, the methods outlined here can still be applied, but low-level safe sets and capacity depletion estimates will need to be specific for each agent. For example, at the beginning of iteration $r$, the set $\mathcal{S}^{L,r,m}_{i,j}$ will contain the trajectory most recently completed by agent $m$ along edge $e_{i,j}$. 
    Similarly, the cost estimate $\hat{\theta}^{r,m}_{i,j,l}$ will contain the smallest amount the capacity state depleted when agent $m$ traversed the edge $e_{i,j}$ during a previous iteration.
\end{remark}

\subsection{Properties}\label{ssec:properties}
We prove that Alg.~\ref{alg:hierarchy} with (\ref{eq:highlevellmpc}) and (\ref{eq:lowlevellmpc}) satisfies iterative feasibility and iterative performance improvement as defined in Sec.~\ref{sec:iterative}.
We first show that an agent (\ref{eq:agentdyn}) in closed-loop with the low-level policy (\ref{eq:lowlevellmpc}) iteratively traversing a single edge $e_{i,j}$ exhibits reduced capacity consumption with each subsequent iteration. 

\begin{theorem}\label{thm:llc}
Consider system (\ref{eq:agentdyn}) controlled by the low-level policy (\ref{eq:lowlevellmpc}) during consecutive traversals of a single edge $e_{i,j}$ in $G$. Let $\mathcal{S}^{L,r}_{i,j}$ be the safe set corresponding to the edge $e_{i,j}$ at the start of iteration $r$ as defined in (\ref{eq:lowlevelss-update}). Then, the optimal control problem (\ref{eq:lowlevellmpc}) solved by the low-level policy will be feasible for all time steps $ 0 \leq t \leq tf$ and iterations $r \geq 1$ from any initial state $x^{r}_{t0} = (c^{r}_{t0}, \bar{\chi}_i)$ where $c^{r}_{l,t0} \in [0, {C}_l - \hat{\theta}^r_{i,j,l}]$, $t0 = \mathcal{I}({\bf{x}}^r,i,j,0)$, and $tf = \mathcal{I}({\bf{x}}^r,i,j,-1)$.
Furthermore, the capacity states drain no more during traversal $r$ than traversal $(r-1)$, so that 
\begin{align*}
    \Omega_l({\bf{x}}^r,i,j) \leq  \Omega_l({\bf{x}}^{r-1},i,j).
\end{align*}
\end{theorem}

\begin{proof}
By construction, $\mathcal{S}^{L, r}_{i,j}$ is non-empty $\forall r \geq 1$. 
There exists a feasible trajectory ${\bf{x}}^{r-1}_{\mathcal{I}({\bf{x}}^{r-1},i,j)} \in \mathcal{S}^{L, r}_{i,j}$ with related input sequence ${\bf{u}}^{r-1}_{\mathcal{I}({\bf{x}}^{r-1},i,j)} \in \mathcal{U}$, where
\begin{align*}
    {\bf{x}}^{r-1}_{\mathcal{I}({\bf{x}}^{r-1},i,j)} = (\tilde{{\bf{c}}}^{r-1}_{\mathcal{I}({\bf{x}}^{r-1},i,j)},  {\bf{\chi}}^{r-1}_{\mathcal{I}({\bf{x}}^{r-1},i,j)}).
\end{align*}
We denote $tr = \mathcal{I}({\bf{x}}^{r-1},i,j,0)$ the time at the beginning of edge $e_{i,j}$ traversal in iteration $(r-1)$.
At time $t=t0$ of iteration $r$, the $N$-step trajectory
\begin{subequations}\label{eq:chic}
 \begin{align}
& {\bf{x}}^{r}_{t0:t0+N | t0} = ({\bf{c}}^{r}_{t0:t0+N | t0}, ~{\bf{\chi}}^{r}_{t0:t0+N | t0}) \\
    &{\bf{c}}^{r}_{t0:t0+N | t0} = \tilde{{\bf{c}}}^{r-1}_{tr : tr +N} + c^{r}_{t0}\\
    &{\bf{\chi}}^{r}_{t0:t0+N | t0} = {\bf{\chi}}^{r-1}_{tr : tr+N} 
\end{align}   
\end{subequations}
and related input sequence
\begin{align}\label{eq:chicu}
    [u_{tr}^{r-1}, u_{tr+1}^{r-1}, \dots, u_{tr+N-1}^{r-1}]
\end{align}
corresponding to the inputs applied by the low-level policy (\ref{eq:lowlevellmpc}) at iteration $(r-1)$
satisfy input and state constraints (\ref{eq:lowlevelstate}), (\ref{eq:lowlevelinput}), (\ref{eq:lowlevelinequalityterm}), (\ref{eq:lowlevelterm}). 

To verify this, we first note that it follows from (\ref{eq:capacitydynamics}) and $g_c \geq 0$ that $c^{r-1}_{\mathcal{I}({\bf{x}}^{r-1},i,j,-1)} = \max_{t \in \mathcal{I}({\bf{x}}^{r-1},i,j)} c^{r-1}_t$. 
By definition (\ref{eq:thetaupdate}) we have that
\begin{align*}
    \hat{\theta}^{r}_{i,j,l} = \Omega_l({\bf{x}}^{r-1},i,j) = \tilde{c}^{r-1}_{\mathcal{I}({\bf{x}}^{r-1},i,j,-1),l}.
\end{align*}
Therefore, $\tilde{{\bf{c}}}^{r-1}_{\mathcal{I}({\bf{x}}^{r-1},i,j)} \in \mathcal{C}$ implies that $\tilde{\bf{c}}^{r-1,i,j}_{tr:tr+N} + c^r_{t0} \in \mathcal{C}$ for any $c^r_{t0,l} \in [0, {C}_l - \hat{\theta}^r_{i,j,l}]$.
It follows with the additional separability of the capacity and non-capacity state dynamics (\ref{eq:capacitydynamics}), (\ref{eq:noncapacitydynamics}) that (\ref{eq:chic}) results from applying the feasible actuation inputs in ($\ref{eq:chicu}$), and (\ref{eq:lowlevelstate}) is satisfied. 
Thus, (\ref{eq:chic}), (\ref{eq:chicu}) is a feasible solution for (\ref{eq:lowlevellmpc}) at time $t=0$ of iteration $r$. 

Assume that at time $t \in \mathcal{I}({\bf{x}}^{r},i,j)$ of iteration $r$, the optimal control problem (\ref{eq:lowlevellmpc}) is feasible, and let ${\bf{x}}^{r,\star}_{t:t+N|t}$ and ${\bf{u}}^{r,\star}_{t:t+N-1|t}$ be the corresponding optimal trajectory and input sequence. The terminal constraints (\ref{eq:lowlevelinequalityterm}), (\ref{eq:lowlevelterm}) dictate that $\exists t^{\star} \in \mathcal{I}({\bf{x}}^{r-1},i,j)$ such that 
\begin{align}
    \chi^{r,\star}_{t+N|t} &= \chi^{r-1}_{t^{\star}}, ~
    c^{r,\star}_{t+N|t} - c^{r}_{t0} \leq \tilde{c}^{r-1}_{t^{\star}} \nonumber
\end{align}
where $(\tilde{c}^{r-1}_{t^{\star}}, \chi^{r-1}_{t^{\star}}) \in \mathcal{S}^{L, r}_{i,j}$.
As the state update in (\ref{eq:agentdynamics}) and (\ref{eq:lowleveldynamics}) are assumed identical we have
\begin{align}
    x^{r}_{t+1} = x^{r, \star}_{t+1|t}.\nonumber
\end{align}
At time $t+1$ of iteration $r$, the input sequence 
\begin{align}\label{eq:aa}
    [u^{r, \star}_{t+1|t}, u^{r,\star}_{t+2|t}, \dots, u^{r,\star}_{t+N|t}, u^{r-1}_{t^{\star}}]
\end{align}
results in the non-capacity and capacity state trajectories
\begin{subequations}\label{eq:bb}
  \begin{align}
    & [\chi^{r, \star}_{t+1|t}, \chi^{r, \star}_{t+2|t}, \dots, \chi^{r, \star}_{t+N-1|t}, \chi^{r-1}_{t^{\star}}, \chi^{r-1}_{t^{\star}+1}] \\
    & [c^{r, \star}_{t+1|t}, c^{r, \star}_{t+2|t}, \dots,   c^{r, \star}_{t+N-1|t},  \nonumber \\
    & ~~~~~~~~~~~~~~~c^{r, \star}_{t+N|t}, c^{r, \star}_{t+N|t} + g_c(\chi^{r-1}_{t^{\star}}, u^{r-1}_{t^{\star}})]
\end{align}  
\end{subequations}
where $c^{r, \star}_{t+N|t} - c^{r}_{t0} \leq \tilde{c}^{r-1}_{t^{\star}}$ and it follows from (\ref{eq:capacitydynamics}) that $c^{r,\star}_{t+N|t} - c^{r}_{t0} + g_c(\chi^{r-1}_{t^{\star}}, u^{r}_{t^{\star}}) \leq \tilde{c}^{r-1}_{t^{\star}+1}$, where $(\tilde{c}^{r-1}_{t^{\star}+1}, \chi^{r-1}_{t^{\star}+1}) \in \mathcal{S}^{L,r}_{i,j}$. 
Therefore, (\ref{eq:aa}) and (\ref{eq:bb}) are a feasible solution for the optimal control problem (\ref{eq:lowlevellmpc}) at time $t+1$ of iteration $r$. 

We have shown that at iteration $r \geq 1$, \textit{i)} the low-level policy is feasible at time $t=0$, and \textit{ii)} if the low-level policy is feasible at time $t$, then it is feasible at time $t+1$. We conclude by induction that the low-level policy (\ref{eq:lowlevellmpc}), (\ref{eq:lowlevellmpc-rec}) will result in a feasible trajectory for all iterations $r \geq 1$ from an initial state $x^{r}_{t0} = (c^{r}_{t0}, \chi^{r}_{t0})$ where $c^{r}_{t0,l} \in [0, {C}_l - \hat{\theta}^r_{i,j,l}]$ and $\chi^{r}_{t0} = \bar{\chi}_i$.

There exists a time step $t^{\star\star} \in \mathcal{I}({\bf{x}}^r,i,j)$ such that
\begin{subequations}\label{eq:cc}
   \begin{align}
    &\chi^{r, \star}_{t^{\star\star}+N|t^{\star\star}} = \chi^{r-1}_{\mathcal{I}({\bf{x}}^{r-1},i,j,-1)}\\
    &c^{r, \star}_{t^{\star\star}+N|t^{\star\star}} - c^{r}_{t0} \leq c^{r-1}_{\mathcal{I}({\bf{x}}^{r-1},i,j,-1)} - c^{r-1}_{\mathcal{I}({\bf{x}}^{r-1},i,j,0)}.
\end{align} 
\end{subequations}
Applying the low-level policy (\ref{eq:lowlevellmpc-rec}) beyond time $t=t^{\star\star}$ results in the closed-loop trajectory at iteration $r$
\begin{align}
    & [\chi^{r}_{t0}, \chi^{r}_{t0+1}, \dots, \chi^{r, \star}_{t^{\star\star}|t^{\star\star}}, \chi^{r, \star}_{t^{\star\star}+1|t^{\star\star}}, \dots, \chi^{r, \star}_{t^{\star\star}+N|t^{\star\star}}] \nonumber\\
    & [c^{r}_{t0}, c^{r}_{t0+1}, \dots, c^{r, \star}_{t^{\star\star}|t^{\star\star}}, c^{r, \star}_{t^{\star\star}+1|t^{\star\star}}, \dots, c^{r, \star}_{t^{\star\star}+N|t^{\star\star}}], \nonumber
\end{align}
where $(c^{r, \star}_{t^{\star\star}+N|t^{\star\star}}, \chi^{r, \star}_{t^{\star\star}+N|t^{\star\star}}) \in \mathcal{N}_j$ and the iteration has concluded.
Therefore $\mathcal{I}({\bf{x}}^r,i,j) = [t0, tf]$, where $tf = t^{\star\star}+N$.
From (\ref{eq:cc}), we conclude that for any successive iterations $r-1$ and $r$, we have  
\begin{align}
    c^{r}_{tf} - c^{r}_{t0} \leq  c^{r-1}_{\mathcal{I}({\bf{x}}^{r-1},i,j,-1) } - c^{r-1}_{\mathcal{I}({\bf{x}}^{r-1},i,j,0) } \nonumber
\end{align}
and therefore
\begin{align}\label{eq:dfa}
    \Omega_l({\bf{x}}^r,i,j) \leq  \Omega_l({\bf{x}}^{r-1},i,j).
\end{align}
We have shown that at iteration $r \geq 1$ of a system (\ref{eq:agentdyn}) controlled by the low-level policy (\ref{eq:lowlevellmpc}), (\ref{eq:lowlevellmpc}) iteratively traversing a single edge $e_{i,j}$, the capacity states drain no more than at iteration $(r-1)$. 
\end{proof}

It follows from Thm.~\ref{thm:llc} that for $\hat{\theta}_{i,j}^{r}$ calculated according to (\ref{eq:thetaupdate}), we have $\hat{\theta}_{i,j}^{r+1} \leq \hat{\theta}_{i,j}^{r}$ for all iterations $r \geq 1$, and similarly, $\hat{\Theta}_{i,j}^{r+1} \subseteq \hat{\Theta}_{i,j}^{r}$ for all iterations $r \geq 1$. 
Furthermore, $\hat{\theta}_{i,j}^{r}$ calculated according to (\ref{eq:thetaupdate}) represents a worst-case capacity state depletion for an agent (\ref{eq:agentdyn}) in closed-loop with the low-level policy (\ref{eq:lowlevellmpc}), (\ref{eq:lowlevellmpc-rec}) along each edge $e_{i,j}$ for all iterations beyond $r$. 
Therefore, using $\hat{\theta}_{i,j}^{r}$ at iteration $r$ ensures that the high-level policy (\ref{eq:highlevellmpc}) plans robustly against all $\theta \in \hat{\Theta}_{i,j}^{r}$ updated according to (\ref{eq:bigthetaupdate}). 
This is stated formally in Thm.~\ref{thm:itfeas}.

\begin{theorem}[Iterative Feasibility]\label{thm:itfeas}
Consider a system (\ref{eq:agentdyn}) in closed-loop with the hierarchical controller outlined in Alg.~\ref{alg:hierarchy}, where $\pi^{H,r}$ is as in (\ref{eq:highlevellmpc}) and $\pi^{L,r}$ as in (\ref{eq:lowlevellmpc}). 
Let $\mathcal{S}^{L,r}_{i,j}$ be the safe set corresponding to each edge $e_{i,j}$ as defined in (\ref{eq:lowlevelss-update}), and $\mathcal{S}^{H,r}$ be the safe set as defined in (\ref{eq:highlevelss-update}). Then, the hierarchical controller outlined in Alg.~\ref{alg:hierarchy} will be iteratively feasible as outlined in Def.~1 from the initial state $x^r_0 = x_S$ for all iterations $r \geq 1$.
\end{theorem}

\begin{proof} 
At event $k=0$ (occurring at time $t=0$) of iteration $r$, the agent is located in the depot, with
\begin{align}
    &x_0^r = x_S = ({\bf{0}}_{L}, \bar{\chi}_D) \nonumber\\
    &x_0^{H,r} = (D, {\bf{0}}_{L}).\nonumber
\end{align}
Assume that the hierarchical controller outlined in Alg.~\ref{alg:hierarchy}, where $\pi^{H,r-1}$ is as in (\ref{eq:highlevellmpc}) and $\pi^{L,r-1}$ as in (\ref{eq:lowlevellmpc}) is feasible at iteration $r-1 \geq 0$. Therefore, $\mathcal{S}^{H,r}$ is non-empty, and $\mathcal{S}^{L,r}_{i,j}$ is non-empty for each edge $e_{i,j}$ in $G$. 
In particular, there exists a feasible high-level trajectory ${\bf{x}}^{H,r-1} \in \mathcal{S}^{H,r}$ with related input $u^{H,r-1}$, where $x^{H,r-1}_0 = x^{H,r}_0 = (D, {\bf{0}}_{L})$. 
At event $k=0$ of iteration $r$, the N-step high-level trajectory
\begin{align}\label{eq:oldhighlevel}
    ({\bf{n}}^{r-1}_{0:N}, {\bf{c}}^{r-1}_{0:N}) \in \mathcal{S}^{H,r},
\end{align}
where ${\bf{c}}^{r-1}_{0:N}$ is the sequence of high-level capacity states at events $k=0$ through $k=N$ of iteration $(r-1)$ and ${\bf{n}}^{r-1}_{0:N}$ the sequence of nodes from events $k=0$ through $k=N$ of iteration $(r-1)$,
associated with the input
\begin{align}\label{eq:oldhighinput}
    {\bf{u}}^{H,r}_{0|0} &= u^{r-1}_0
\end{align}
and capacity states
\begin{align}\label{eq:oldcapstates}
    {\bf{c}}^{r}_{0:N|0} = {\bf{c}}^{r-1}_{0:N}
\end{align}
is a feasible high-level trajectory which satisfies all constraints (\ref{eq:firsthighlevelconstraint}) - (\ref{eq:lastlmpcconstraint}). 
Note particularly that, as follows from Thm.~\ref{thm:llc}, $\hat{\theta}^{r-1}_{D, n^{r-1}_1} \geq \hat{\theta}^r_{D, n^{r-1}_1}$, and therefore constraints (\ref{eq:37})-(\ref{eq:39}) will be satisfied at iteration $r$ with (\ref{eq:oldcapstates}).
Thus, (\ref{eq:highlevellmpc}) is feasible at event $k=0$ of iteration $r$. 

We denote the optimal $N$-step high-level state trajectory as determined by the optimal control problem (\ref{eq:highlevellmpc}) at event $k=0$ of iteration $r$ as
\begin{align}\label{eq:highlevelref}
    {\bf{x}}^{H,r,\star}_{0:N|0} = ({\bf{n}}^{H,r,\star}_{0:N|0}, {\bf{c}}^{H,r,\star}_{0:N|0})
\end{align}
where $n^{H,r,\star}_{1|0} = j$ is sent as the high-level reference node for the low-level controller to track. 
Since (\ref{eq:highlevelref}) must satisfy constraint (\ref{eq:39}), it follows that
\begin{align}
    C_l \geq \hat{\theta}^r_{D,j,l},~ \forall l \in [1,L].\nonumber
\end{align}
Therefore, the agent state $x^r_0$ is such that it satisfies the initial state conditions in Thm.~\ref{thm:llc}, and the low-level policy (\ref{eq:lowlevellmpc}),(\ref{eq:lowlevellmpc-rec}) will be feasible for all time steps $t \in [0, \mathcal{I}({\bf{x}}^r,D,j)]$ and bring the agent to node $v_j$.

Assume at event $k$ (occurring at time $t$) of iteration $r$, where the agent is at node $v_i$ with state $x^r_t = (c^r_t, \chi^r_t)$ so that
\begin{align}
    &x^r_t \in \mathcal{N}_i\nonumber\\
    &g(x^r_t) = (i, c^r_k)\nonumber
\end{align}
with capacity states $c^r_k = c^r_t$, the high-level optimal control problem (\ref{eq:highlevellmpc}) is feasible. Let 
\begin{align}\label{eq:ffdd}
    {\bf{x}}^{H,r,\star}_{k:k+N|k} = (  {\bf{n}}^{r,\star}_{k:k+N|k},  {\bf{c}}^{r,\star}_{k:k+N|k})
\end{align}
be the optimal high-level trajectory containing the high-level capacity state trajectory ${\bf{c}}^{r,\star}_{k:k+N|k}$, corresponding to input $u^{H,r\star}_k$. 
The terminal constraints (\ref{eq:cinequality}), (\ref{eq:lmpcinequality}) dictate that $\exists k^{\star} \geq 0$ such that
\begin{align}
    n^{r,\star}_{k+N|k} = n^{p}_{k^{\star}},~ c^{r,\star}_{k+N|k} \leq c^p_{k^{\star}},\nonumber
\end{align}
where $( n^{p}_{k^{\star}}, c^p_{k^{\star}} ) \in \mathcal{S}^{H,r}$, and $p \leq r$. 

Since (\ref{eq:highlevelref}) must satisfy constraint (\ref{eq:39}), it follows that
\begin{align}
    C_l \geq \hat{\theta}^r_{i,n^{r,\star}_{k+1|k},l},~ \forall l \in [1,L],\nonumber
\end{align}
and therefore the low-level optimal control problem will be feasible for all time steps $t \in \mathcal{I}({\bf{x}}^r,i,n^{r,\star}_{k+1|k})$ and bring the agent to node $n^{r,\star}_{k+1|k}$. 

The planned capacity state trajectory in (\ref{eq:ffdd}) utilizes the high-level model estimate $\hat{\theta}^r$ to predict how capacity states will drain across the planned set of tasks.
However, as shown in Thm.~\ref{thm:llc}, the agent's true capacity state consumption during iteration $r$ may be less than predicted by $\hat{\theta}^r$ (as will be captured in $\hat{\theta}^{r+1}$ when updated according to (\ref{eq:thetaupdate})).
Thus the high-level state at event $k+1$, when the agent has reached node $n^{r,\star}_{k+1|k}$, is
\begin{align}
    x^{H,r}_{k+1} = (n^{r,\star}_{k+1|k}, c^r_{k+1})\nonumber
\end{align}
where $c^r_{k+1} \leq c^{r,\star}_{k+1|k}$.

At event $k+1$ of iteration $r$, the input matrix $u \in \{0,1\}^{|V|\times|V|}$ such that
\begin{align}\label{eq:rere}
    u_{i,j \in \mathcal{V}} = \begin{cases}
    0 & i = n^r_k,~ j = n^r_{k+1} \\
    1 & i = n^{r\star}_{k+N|k},~ j = n^{p}_{k^{\star}+1} \\
        u^{H,r}_{k,i,j} & \text{else}
    \end{cases}
\end{align}
and capacity state values
\begin{align}\label{eq:tyty}
    [c^{r,\star}_{k+1|k}, c^{r,\star}_{k+2|k}, \dots, c^{r,\star}_{k+1|k}, c^{r,\star}_{k+1|k} + \hat{\theta}^r_{i,j}]
\end{align}
where $c^{r,\star}_{k+1|k} + \hat{\theta}^r_{i,j} \leq c^p_{k^{\star}+1}$ and $c^p_{k^{\star}+1} \in \mathcal{SS}^{H,r}$,
satisfy all constraints of the high-level optimal control problem (\ref{eq:highlevellmpc}).
Note that the input matrix in (\ref{eq:rere}) plans an $N$-step path along the graph $G$ that follows the node trajectory
\begin{align}
    [n^{r\star}_{k+1|k}, n^{r\star}_{k+2|k}, \dots, n^{r\star}_{k+N|k},  n^{p}_{k^{\star}+1}].\nonumber
\end{align}
Therefore, (\ref{eq:rere}) and (\ref{eq:tyty}) above are a feasible solution for the high-level optimal control problem at event $k+1$ of iteration $r$. 

We have shown that at iteration $r\geq 1$, if the hierarchical controller as outlined in Alg.~\ref{alg:hierarchy} was feasible at iteration $r-1$, then \textit{i)} the hierarchical controller outlined in Alg.~\ref{alg:hierarchy} is feasible at event $k=0$, and $ii)$ if the hierarchical controller is feasible at event $k$, then it is feasible at event $k+1$. We conclude by induction that the hierarchical controller using (\ref{eq:lowlevellmpc}), (\ref{eq:highlevellmpc}) will result in a feasible trajectory for all iterations $r\geq 1$ from the initial state $x^r_0 = x_s$.
\end{proof}

\begin{theorem}[Iterative Cost Improvement]
Consider a system (\ref{eq:agentdyn}) in closed-loop with the hierarchical controller outlined in Alg.~\ref{alg:hierarchy}, where $\pi^{H,r}$ is as in (\ref{eq:highlevellmpc}) and $\pi^{L,r}$ as in (\ref{eq:lowlevellmpc}). 
Let $\mathcal{S}^{L,r}_{i,j}$ be the safe set corresponding to each edge $e_{i,j}$ as defined in (\ref{eq:lowlevelss-update}), and $\mathcal{S}^{H,r}$ be the safe set as defined in (\ref{eq:highlevelss-update}). Then, the hierarchical controller outlined in Alg.~\ref{alg:hierarchy} will exhibit iterative performance improvement as outlined in Def.~2 from the initial state $x^r_0 = x_S$ for all iterations $r \geq 1$.
\end{theorem}
\begin{proof}
By definition of the iteration cost, the closed-loop performance of the hierarchical controller at iteration $r$ in Alg.~\ref{alg:hierarchy} is given by
\begin{align*}
    V_n({\bf{x}}^r) = V_n^H(g_{\tau}({\bf{x}}^r) = V_n^H({\bf{x}}^{H,r})\nonumber
\end{align*}
and captures how many unique tasks the agent completes along the trajectory ${\bf{x}}^r$ before returning to the depot. We note this is the same performance function the high-level policy (\ref{eq:highlevellmpc}) aims to maximize at each event $k$.
It follows from Thm.~\ref{thm:llc} and Thm.~\ref{thm:itfeas} that the high-level optimal control problem (\ref{eq:highlevellmpc}) always plans $N$-step trajectories of nodes and capacity state values which the low-level policy (\ref{eq:lowlevellmpc}) can accurately track. It therefore suffices in this proof to only consider the performance of the high-level policy (\ref{eq:highlevellmpc}), recognizing that the high-level trajectories are inherently designed to be feasible for the low-level policy. 

The proof follows arguments from Thm. 2 in \cite{lmpc}.
Consider a closed-loop trajectory ${\bf{x}}^{r-1}$ produced by Alg.~\ref{alg:hierarchy} at iteration $(r-1)$. By definition of the iteration cost, we have
\begin{subequations}
\begin{align}
    V_n({\bf{x}}^{r-1}) &= V_n^H({\bf{x}}^{H,r-1}) \nonumber \\
    &\leq V_n^H({\bf{x}}^{H,r-1}_{0:N-1}) +Q^{H, r}(x^{H,r-1}_N) \nonumber \\
    & \leq \max_{{\bf{x}}^H_{0:N}} ~ V_n^H({\bf{x}}^{H}_{0:N-1}) + Q^{H, r}(x^{H}_N)\label{eq:lmpcreform}\\
    & = V^{H, r}_{n, 0 \rightarrow N}(x_0^{H,r}), \label{eq:lowerbound}
\end{align}
\end{subequations}
where $V^{H, r}_{n, 0 \rightarrow N}(x_0^{H,r})$ is the solution to (\ref{eq:highlevellmpc}), the predicted closed-loop cost for iteration $r$ at event $k=0$ from initial state $x_0^{H,r} = x_0^{H,r-1}$.  
Note that, as shown in Thm.~\ref{thm:itfeas}, ${\bf{x}}^{H,r-1}_{0:N}$ is a feasible solution (\ref{eq:lmpcreform}). 
Now, notice that 
    \begin{align}
    & V^{H, r}_{n, 0 \rightarrow N}(x_k^{H,r}) = \nonumber\\
    &=  \max_{{\bf{x}}^H_{k:k+N|k}} ~ V_n^H({\bf{x}}^{H}_{k:k+N-1|k}) + Q^{H,r}(x^H_{k+N|k}) \label{eq:1}\\
    & = V_n^H(x^{H,r,\star}_{k|k}) + V_n^H({\bf{x}}^{H,r,\star}_{k+1:k+N-1|k}) + Q^{H,r}(x^{H,r,\star}_{k+N|k}) \label{eq:2}\\
    & = V_n^H(x^{H,r,\star}_{k|k}) + V_n^H({\bf{x}}^{H,r,\star}_{k+1:k+N-1|k}) +  V_n^H({\bf{x}}^{H,p}_{k\star:}) \label{eq:pkstar}\\
    &= V_n^H(x^{H,r,\star}_{k|k}) + V_n^H({\bf{x}}^{H,r,\star}_{k+1:k+N-1|k})  + V_n^H(x^{H,p}_{k\star}) + \nonumber\\
    & ~~~~~~~~~~~~~~~~~~~~~~~~~~~~~~~~~~~~~~~~~~~~+ Q^{H,r}(x^{H,p}_{k\star+1}) \nonumber\\
    &\leq  V_n^H(x^{H,r,\star}_{k|k}) + V^{H, r}_{n, 0 \rightarrow N}(x_{k+1|k}^{H,r,\star})\label{eq:takeaway} \\
    &\leq  V_n^H(x^{H,r,\star}_{k|k}) + V^{H, r}_{n, 0 \rightarrow N}(x_{k+1}^{H,r}),\label{eq:takeaway2} 
\end{align}
Where $p,~ k\star$ in (\ref{eq:pkstar}) are as defined in (\ref{eq:highcosttogo}).
Note that we can break the cost $V_n^H$ evaluated on a closed-loop trajectory ${\bf{x}}^{H}$ into the sum of costs evaluated on components of the trajectory (as in (\ref{eq:1}) and (\ref{eq:2})) because the high-level optimal control problem enforces in constraint (\ref{eq:prevpathconstraint}) that planned paths do not contain nodes already traversed in the current iteration. 
(\ref{eq:takeaway}) states that the predicted closed-loop cost for iteration $r$ at event $k$ is bounded above by the sum of the current stage cost and the closed-loop cost at the predicted state at event $k+1$. This bound is further improved in~(\ref{eq:takeaway2}), since (as shown in Thm.~\ref{thm:llc} and Thm.~\ref{thm:itfeas}) $c_{k+1|k}^{H,r,\star}\geq c_{k+1}^{H,r}$, potentially allowing for a better route to be planned at event $k+1$ than was expected at event $k$. 
Iteratively applying (\ref{eq:takeaway2}) at event $k=0$ gives
    \begin{align}
        V^{H, r}_{n, 0 \rightarrow N}(x_0^{H,r}) & \leq V_n^H(x_0^{H,r}) + V^{H, r}_{n, 0 \rightarrow N}(x_1^{H,r}) \nonumber\\
    & \leq V_n^H({\bf{x}}_{0:1}^{H,r}) + V^{H, r}_{n, 0 \rightarrow N}(x_2^{H,r}) \nonumber\\
    & \leq \lim_{k \rightarrow \infty} \big( V_n^H({\bf{x}}_{0:k-1}^{H,r}) + V^{H, r}_{n, 0 \rightarrow N}(x_k^{H,r}) \big)\label{eq:lastlabel}.
    \end{align}
The agent is constrained to return to the depot (node $D$) by the end of the task, and thus by definition of (\ref{eq:hierarchicalcost}),
\begin{align}
    \lim_{k \rightarrow \infty} V_{n, 0 \rightarrow N}^{H,r} ( x^{H,r}_{k}) = V_{n, 0 \rightarrow N}^{H,r} ((D, \cdot)) = 0. \label{eq:blep}
\end{align}
From (\ref{eq:lastlabel}), 
\begin{align}
    V^{H, r}_{n, 0 \rightarrow N}(x_0^{H,r}) & \leq \lim_{k \rightarrow \infty}V_n^H({\bf{x}}_{0:k-1}^{H,r}) = V_n^H({\bf{x}}^{H,r}),\label{eq:upper}
\end{align}
and from (\ref{eq:lowerbound}), (\ref{eq:upper}), we conclude
\begin{align}
     V_n^H({\bf{x}}^{H,r-1}) \leq V^{H, r}_{n, 0 \rightarrow N}(x_0^{H,r}) \leq V_n^H({\bf{x}}^{H,r}), \nonumber
\end{align}
and therefore $V_n({\bf{x}}^{r-1}) \leq V_n({\bf{x}}^{r})$.
Therefore, the Alg.~\ref{alg:hierarchy} results in at least as many tasks being visited during iteration $r$ than iteration $r-1$, and the hierarchical controller exhibits iterative performance improvement from the initial state $x_0^r = x_s$.
\end{proof}

\section{Example}
We demonstrate the effectiveness of the proposed learning hierarchical control approach in a simple simulation example.

\subsection{Agent}
We consider a single agent ($M=1$) and a task graph containing $|V|=7$ nodes, as depicted in Fig.~\ref{fig:PathAssignment}. 
At time step $t$ of iteration $r$, the agent is modeled with state 
\begin{align}
    x^r_t = (z^r_t, y^r_t, \theta^r_t, v^r_t, soc^r_t, \tau^r_t),
\end{align}
where $z^r_t$ and $y^r_t$ denote the agent's position, $\theta^r_t$ the steering angle, $v^r_t$ the velocity, $soc^r_t$ the battery state of charge, and $\tau^r_t$ how much time has passed since the agent left the depot node $\mathcal{N}_1$ at the beginning of the iteration. 
The agent state evolves according to the kinematic bicycle model
\begin{subequations}\label{eq:exampleagentmodel}
    \begin{align}
        z^r_{t+1} &= z^r_t + v^r_t \cos{\theta^r_t} dt \\
        y^r_{t+1} &= y^r_t + v^r_t \sin{\theta^r_t} dt \\
        \theta^r_{t+1} &= \theta^r_t + \delta^r_t dt \\
        v^r_{t+1} &= v^r_t + a^r_t dt \\
        soc^r_{t+1} &= soc^r_t + \alpha v^r_t dt \\
        \tau^r_{t+1} &= \tau^r_t + dt,
    \end{align}
\end{subequations}
where the parameter $\alpha$ determines the battery discharge rate; here we use $\alpha = 1.6$ The two actuation inputs are the steering rate $\delta^r_t$ and acceleration $a^r_t$. We use a time step of $dt = 0.1$. 
The agent states and inputs are constrained as
\begin{subequations}\label{eq:exampleconstraints}
    \begin{align}
        \theta^r_t &\in [-\pi, ~\pi] \\
        v^r_t &\in [0, ~5] \label{eq:examplevelconstraint}\\
        soc^r_t &\in [0, ~100] \\
        \tau^r_t &\in [0,~ 120] \\
        \delta^r_t &\in [-2, ~2] \\
        a^r_t &\in [-2, ~2].
    \end{align}
\end{subequations}

There are two capacity states: the state of charge of the agent and how much time has passed since the beginning of the iteration
$$c^r_t = (soc^r_t, \tau^r_t).$$ At the beginning of each iteration $r$, the agent is located at the depot node $\mathcal{N}_1$ and each capacity state is initialized at $0$, i.e. $c^r_0 = (0,0)$, $x^r_0 = (1,0,0)$. 
The capacity state constraints indicate the agent must visit as many nodes as possible without depleting its battery and return to the depot by the allotted latest time $T_{max} = 120$. Again, we model the battery depletion as increasing along the route.


\subsection{Controller Formulation}
The agent is tasked with completing as many tasks as possible while satisfying all state constraints (\ref{eq:exampleconstraints}), and is controlled using the hierarchical data-driven control structure outlined in Alg.~\ref{alg:hierarchy}.

The high-level controller (\ref{eq:highlevellmpc}) explicitly models the agent's most recent node location and capacity state values in an event-driven fashion, updating whenever the agent reaches a new node. At event $k$ of iteration $r$, the high-level state is
\begin{align}
    x^{H,r}_k = (n^r_k, soc^r_k, \tau^r_k),
\end{align}
where $n^r_k$ is the $k$th node visited by the agent during iteration $r$, $soc^r_k$ is the agent's state of charge when it reaches node $n^r_k$, and $\tau^r_k$ is how much time passed between the agent departing the depot and reaching node $n^r_k$. 
The low-level controller (\ref{eq:highlevellmpc}) uses the agent model in (\ref{eq:exampleagentmodel}). 

As described in Sec.~\ref{ssec:initialdataset}, before beginning iteration $r=1$ we assume the availability of an initial feasible agent trajectory ${\bf{x}}^{0}$ beginning and ending at the depot, along with feasible trajectories along all edges $e_{i,j}$ in $G$. 
These low-level trajectories are created for each edge $e_{i,j}$ by solving (\ref{eq:lowleveliter}), using a conservative velocity constraint $V_{max} = 2$. Note this conservative velocity constraint is enacted only during the creation of the initialization trajectories; at runtime of Alg.~\ref{alg:hierarchy}, the velocity constraint is as in (\ref{eq:examplevelconstraint}).
The stage cost $h$ is set as 
\begin{equation}\label{eq:examplestagecost}
        h(x,j) \begin{cases}
			= 0 & x \in \mathcal{N}_j\\
            1 & \text{else}.
		 \end{cases}
\end{equation}
With this stage cost (\ref{eq:examplestagecost}), the low-level controller (\ref{eq:lowlevellmpc}) mimics a minimum-time controller, planning agent trajectories that minimize the number of time steps before the reference node $\mathcal{N}_j$ is reached. 

Given the initial feasible trajectories, the control parameters $\hat{\theta}^1$ (\ref{eq:thetaupdate}), $\hat{\Theta}^1$ (\ref{eq:bigthetaupdate}), $\mathcal{S}^{H,1}$ (\ref{eq:highlevelss-update}) and $\mathcal{S}^{L,1}$ (\ref{eq:lowlevelss-update}) are calculated for utilization in the high-level (\ref{eq:highlevellmpc}) and low-level (\ref{eq:lowlevellmpc}) MPC controllers. 
At each event $k$ of iteration $r$, the high-level controller (\ref{eq:highlevellmpc}) searches for paths of length $N_H$ from the current node to a node in the safe set $\mathcal{S}^{H,r}$ from which it can safely return to the depot while satisfying capacity constraints according to the estimate $\hat{\theta}^{r}$. The next assigned node is then tracked by the low-level controller (\ref{eq:lowlevellmpc}), which searches for a low-level agent trajectory of length $N_L$ from the current agent state to a state in the safe set $\mathcal{S}^{L,r}$. 
After each additional iteration is completed and the agent has returned to the depot, the parameters are updated for utilization at the subsequent control iteration.
Here we use a different MPC horizon for the high-level controller ($N_H = 3$) than the low-level controller ($N_L = 15$). 

\subsection{Simulation results}

The hierarchical control algorithm is run for five iterations.
Figure~\ref{fig:PathAssignment} depicts the resulting high-level path assignments at three different iterations, where iteration $r=0$ corresponds to the initialization assignment. In iteration $r=0$, the agent visits two unique nodes before returning to the depot. In iteration $r=1$, the first iteration produced by Alg.~\ref{alg:hierarchy},
the high-level controller (\ref{eq:highlevellmpc}) is able to find a more productive task assignment which visits four unique nodes while still satisfying all capacity state constraints. As additional low-level trajectories are collected during each iteration and capacity depletion estimates are subsequently updated, the high-level controller is able to find a final route assignment that visits five unique nodes.
Figure~\ref{fig:CompletedTasks} shows how the number of completed tasks grows with additional task iterations, before converging to a particular task assignment beyond iteration $r=2$ (node $\mathcal{N}_5$ is never able to be visited while satisfying state constraints). In future studies, we will further explore these convergence properties. 


\begin{figure}
    \centering
    \includegraphics[width=0.35\textwidth]{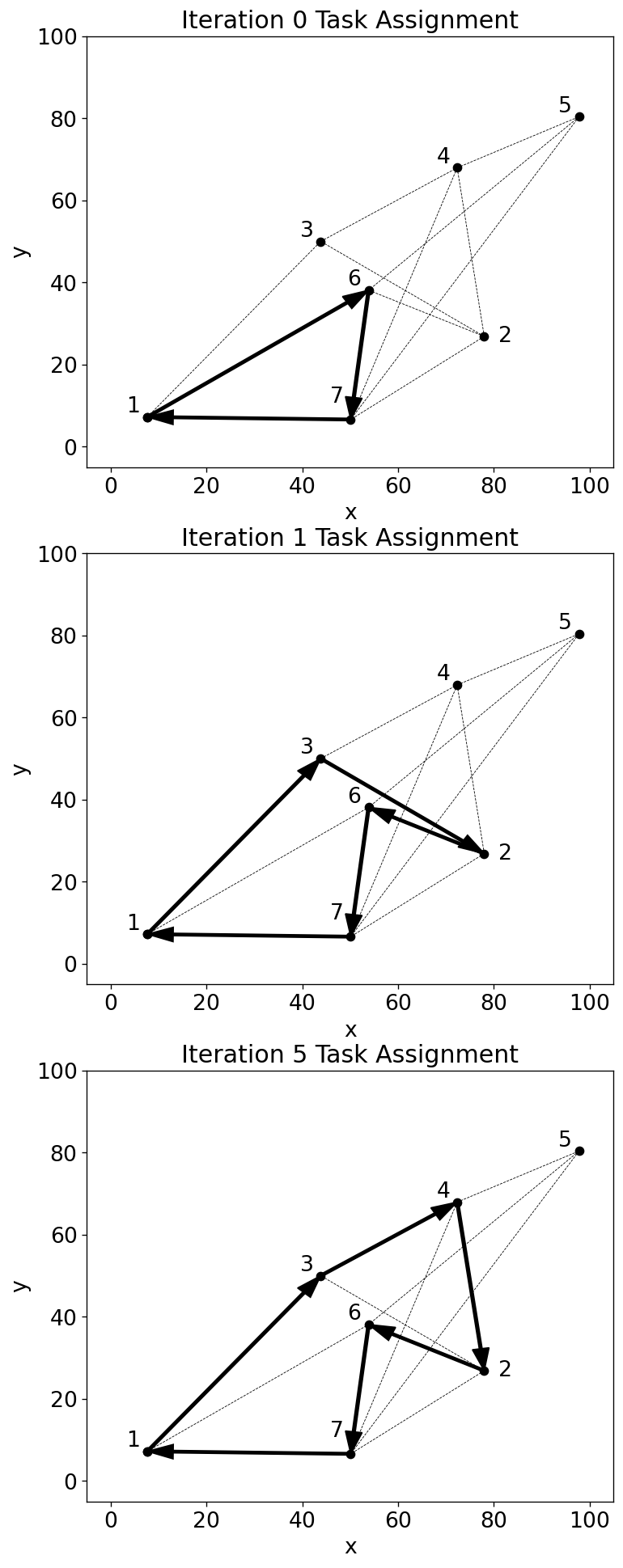}
    \caption{With additional iterations, the high-level controller is able to improve on its initialization assignment (iteration $0$) and plan more productive task assignments, eventually converging to the task assignment shown in the bottom figure (iteration $5$).}
    \label{fig:PathAssignment}
\end{figure}

\begin{figure}
    \centering
    \includegraphics[width=0.32\textwidth]{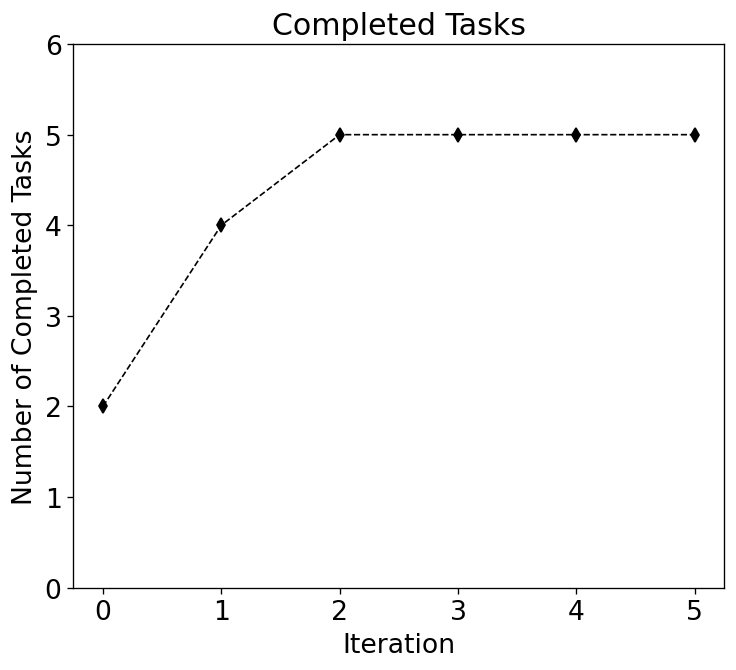}
    \caption{The number of completed tasks increases during the first 3 iterations, before leveling out to a final task assignment.}
    \label{fig:CompletedTasks}
\end{figure}

\begin{figure}
    \centering
    \includegraphics[width=0.37\textwidth]{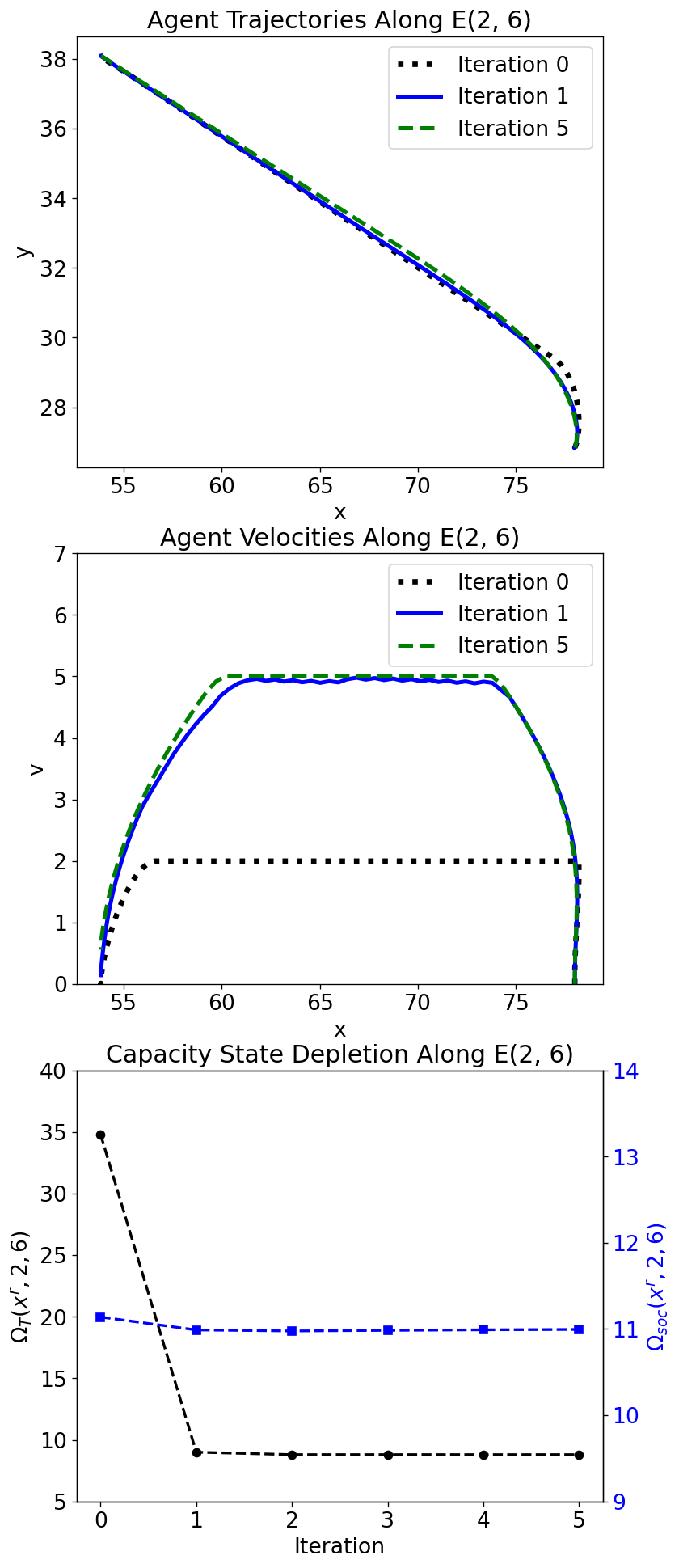}
    \caption{At each traversal of the edge $e_{2,6}$, the low-level controller searches for a trajectory that minimizes the stage cost (\ref{eq:examplestagecost}) without increasing capacity state depletions along the edge. The controller is safely able to guide the agent to increased velocities at subsequent iterations, reducing the time required to traverse the edge while maintaining the agent's state of charge.}
    \label{fig:26EdgePlot}
\end{figure}

Note that high-level routing improvements result from both improved high-level task assignments \textit{and} improved low-level motion planning.
Figure~\ref{fig:26EdgePlot} depicts the agent's trajectories along edge $e_{2,6}$ at three different iterations, where iteration $r=0$ corresponds to the initialization trajectory for the edge. The formulation of the low-level controller (\ref{eq:lowlevellmpc}) allows it to safely improve the closed-loop trajectory with respect to (\ref{eq:examplestagecost}) with each subsequent traversal of any particular edge. As shown in Fig.~\ref{fig:26EdgePlot} for $e_{2,6}$, the controller is able to safely increase the agent's velocity along the edge from $V_{max} = 2$ during the initialization trajectory to $V_{max} = 5$ and, as a result, the time required to traverse the edge improves. Critically, the controller formulation and especially constraint (\ref{eq:lowlevelinequalityterm}) ensures that the battery depletion along the edge never increases between iterations. As shown in Thm.~\ref{thm:itfeas}, this property guarantees the iterative feasibility property of the hierarchical controller.

The high-level controller utilizes the improved low-level trajectories to improve routing assignments.
At event $k=1$ of iteration $r=1$, the agent is at node $\mathcal{N}_3$ (see Fig.~\ref{fig:PathAssignment}). Note that from node $\mathcal{N}_3$, the path $\mathcal{N}_3 - \mathcal{N}_4-\mathcal{N}_2-\mathcal{N}_6$ is a path of length $N_H = 3$ to node $\mathcal{N}_6$, a node traversed in iteration $0$ and therefore in $\mathcal{S}^{H,1}$. However, at the beginning of iteration $r=1$, the capacity state depletion estimate $\hat{\theta}^{1}_{2,6}$ is still too high to allow the high-level controller to plan this route while satisfying the agent's capacity constraints (\ref{eq:exampleconstraints}). 
As a result, the less productive (but feasible, according to $\hat{\theta}^1$) assignment $\mathcal{N}_3 - \mathcal{N}_2-\mathcal{N}_6-\mathcal{N}_7$ is planned at event $k=1$. 
After this traversal of edge $e_{2,6}$ in iteration $r=1$, the capacity state depletion estimate improves (as depicted in Fig.~\ref{fig:26EdgePlot}); at event $k=1$ of iteration $r=2$, the high-level controller is able to safely plan the more productive route $\mathcal{N}_3 - \mathcal{N}_4-\mathcal{N}_2-\mathcal{N}_6$.





This simulation example serves to demonstrate the effectiveness and feasibility of the proposed data-driven hierarchical control framework presented here. In future studies, we will evaluate the framework on larger graphs $G$, exploring the associated computational complexity with solving (\ref{eq:highlevellmpc}) as well as the convergence properties of the data-driven hierarchical controller.

\section{Conclusion}
In this paper, a novel data-driven hierarchical control scheme for task assignment and control of capacity-constrained multi-agent systems is presented. 
The proposed algorithm uses tools from iterative learning control to leverage collected system trajectory data in order to improve estimates of capacity usage during task execution.
Each level of the control hierarchy uses a data-driven MPC policy to maintain bounded computational complexity at each calculation of a low-level actuation input or high-level task assignment. 
We demonstrated how to design a low-level controller with fixed, short horizon that still achieves capacity consumption reduction with each iteration, and show how to design a high-level controller with a fixed, short horizon that safely and dynamically allocates tasks to agents.
We prove that the resulting control algorithm is iteratively feasible and exhibits iterative control performance improvement with each subsequent task iteration.
Future studies will consider how to adjust the framework with the addition of low-level model uncertainty, as well as low-level models that vary with temporal patterns.

\newpage
\section*{APPENDIX}

\subsection{Main Problem}
\begin{subequations}\label{eq:mainproblemappendix}
\begin{align}
    \max_{v,x, u, T} ~ & \sum_{m \in \mathcal{M}} \sum_{i, j \in \mathcal{V}} v^{m}_{i,j} \label{eq:cost}\\
    \text{s.t.}~~ & \sum_{m \in \mathcal{M}}\sum_{i \in \mathcal{V}} v^{m}_{i,j} \leq 1 ~~~~~~~~~~~~~~~~~~~~~\forall j \in \mathcal{V}  \label{eq:bleaving}\\
    & \sum_{m \in \mathcal{M}}\sum_{j \in \mathcal{V}} v^{m}_{i,j} \leq 1 ~~~~~~~~~~~~~~~~~ \forall i \in \mathcal{V} \setminus D \label{eq:barriving}\\
    & \sum_{j \in \mathcal{V}} v^{m}_{i,j} = \sum_{j \in \mathcal{V}} v^{m}_{j,i} ~~~~~~~~~~\forall i \in \mathcal{V}, ~m \in \mathcal{M} \label{eq:bleaving-without-arriving} \\
    & v^m_{i,j} = 0 ~~~~~~~~~\forall i \in \mathcal{V}, ~ j \notin N(i),~ m \in \mathcal{M} \label{eq:onlyneighbors}\\
    & x^m_{t+1} = f(x^m_t, u^{m}_t) ~~~~~\forall t \in [0,\bar{T}], ~m \in \mathcal{M}\label{eq:agentdynamics}\\
    & x^m_{t+1} \in \mathcal{X}, ~ u^{m}_t \in \mathcal{U} ~~~~ \forall t \in [0,\bar{T}], ~m \in \mathcal{M} \label{eq:agentconstraints}\\
    & x^m_0 \in \mathcal{N}_D ~~~~~~~~~~~~~~~~~~~~~~~~~~~~\forall m \in \mathcal{M} \label{eq:start-node}\\
    & x^m_{T^m} \in \mathcal{N}_D ~~~~~~~~~~~~~~~~~~~~~~~~~~\forall m \in \mathcal{M} \label{eq:end-node} \\
    & x^m_{T^m_j} \in (1 - \sum_{i \in \mathcal{V}} v^{m}_{i,j})\mathcal{N}_D + \sum_{i \in \mathcal{V}} v^{m}_{i,j} \mathcal{N}_j  \nonumber\\
    & ~~~~~~~~~~~~~~~~~~~~~~~~~~~~~~~ \forall j \in \mathcal{V}, m \in \mathcal{M}  \label{eq:nodetimedynamics}\\
    & T^m_j \geq T^m_iv^m_{i,j}~~\forall i \in \mathcal{V}, ~j \in \mathcal{V}\setminus D, ~ m \in \mathcal{M} \label{eq:timeconstraint1}\\
    & T^m_j \leq T^mv^m_{i,j} ~~\forall i \in \mathcal{V}, ~j \in \mathcal{V}\setminus D, ~ m \in \mathcal{M}  \label{eq:timeconstraint2} \\
    & T^m \leq \bar{T} ~~~~~~~~~~~~~~~~~~~~~~~~~~~~~~\forall m\in \mathcal{M} 
\end{align}
\end{subequations}
where $\mathcal{M} = \{1, 2, \dots, M \}$ and $\mathcal{V} = \{D, 2, 3, \dots, |V|\}$.
The first decision variable ${{v}} \in \mathbb{B}^{|V|\times |V| \times |M|}$ contains the binary decision variables $v^m_{i,j}$ that describe whether agent $m$ is routed along edge $e_{i,j}$. The objective function therefore maximizes how many total nodes are visited collectively by the agents. 
The second decision variable ${x} \in \mathbb{R}^{n\times \bar{T} \times m}$ tracks all states $x^m_t$ of each $m$th agent at each time step $t \in [0, \bar{T}]$ under the inputs ${{u}} \in \mathbb{R}^{u\times (\bar{T}-1) \times m}$. ${{T}} \in \mathbb{R}^{|V| \times M}$ contains the time steps $T^m_j$ when agent $m$ reaches node $v_j$. If agent $m$ never reaches node $v_j$, $T^m_j = 0$. 
The optimal control problem (\ref{eq:mainproblem}) simultaneously searches for a path ${{v}^{\star}}$ that maximizes the number of nodes reached by agents, and a sequence of continuous actuation inputs ${{u}^{\star}}$ that result in the agents visiting the nodes before returning to the depot. 
Constraints (\ref{eq:bleaving}) - (\ref{eq:barriving}) ensure that each node is visited by at most one agent, and constraint (\ref{eq:bleaving-without-arriving}) ensures that the planned path is continuous.
Constraints (\ref{eq:agentdynamics}) - (\ref{eq:agentconstraints}) describe the dynamics and system constraints of the agent, including all capacity state constraints. 
The constraints (\ref{eq:start-node})-(\ref{eq:end-node}) indicate that at time $t=0$ and $t=T^m$, corresponding to the start and end of the task, agent $m$ is at the depot. Constraint (\ref{eq:nodetimedynamics}) ensures that the agent reaches each node $v_j$ assigned to it at some time $T_j^m$. Lastly, constraints (\ref{eq:timeconstraint1})-(\ref{eq:timeconstraint2}) ensure that the agent visits assigned nodes in the order indicated by ${{v}^{\star}}$.
This problem formulation is based on the Flow Model formulation of the capacitated vehicle routing problem. 
For more details on this formulation, we refer to \cite{routing}.

\subsection{High-level Problem}
\begin{subequations}\label{eq:highlevelproblem-iter-appendix}
\begin{align}
    {{u}}^{r,\star}(x^{H,r}_k&) =   \\
    arg\max_{{u}, {c}} ~ & \sum_{i, j \in \mathcal{V}} u_{i,j} \label{eq:cost}\\
    s.t.~~ & \sum_{i \in \mathcal{V}} u_{i,j} \leq 1 ~~~~~~~~~~~~~~~~~~~~ \forall j \in \mathcal{V}  \label{eq:leaving}\\
    & \sum_{j \in \mathcal{V}} u_{i,j} \leq 1 ~~~~~~~~~~~~~~~~~ \forall i \in \mathcal{V} \setminus D \label{eq:arriving}\\
    & \sum_{j \in \mathcal{V}} u_{i,j} \leq \sum_{j \in \mathcal{V}} u_{j,i}, ~~~~~~~~~~~~~\forall i \in \mathcal{V} \label{eq:leaving-without-arriving} \\
    & \sum_{j \in \mathcal{V} \setminus D} u_{n^{H,r}_k,j} = 1 \label{eq:maxvehicles, leave from depot}\\
    & \sum_{j \in \mathcal{V} \setminus D} u_{j,D} = 1 \label{eq:returntodepot}\\
    & u_{i,j} = 0~~ \forall i \in \mathcal{V},~ j \notin N(i) \\
    & c_{i,j,l} = 0~~\forall i \in \mathcal{V}, ~j \notin N(i),~ l \in [1,L] \\
    & c_{n^{r}_k, j, l} = c^r_{l,k} ~~~~~~~~~~~\forall j \in \mathcal{V},~l\in [1,L] \label{eq:edge_initial} \\
    & \sum_{i \in \mathcal{V}} c_{i,j,l} + \hat{\theta}^{r}_{i,j,l} = \sum_{i \in \mathcal{V}} c_{j,i,l} \nonumber \\
    & ~~~~~~~~~~~~~~~~~~~~~~~~~~ \forall j \in \mathcal{V},~  l \in [1,L], \label{eq:edge_dynamics}\\
    & c_{i,j,l} \geq \theta_{i,j,l} u_{i,j}~ \nonumber \\
    & ~~~~~~~~\forall i \in \mathcal{V}, ~j \in \mathcal{V}, ~ l \in [1,L],~\theta \in \hat{\Theta}^{r}\label{eq:edge_min}\\
    & c_{i,j,l} \leq (C_l - \theta_{i,j,l})u_{i,j} \nonumber\\
    & ~~~~~~~~ \forall i \in \mathcal{V}, ~j \in \mathcal{V} , ~l \in [1,L], ~\theta \in \hat{\Theta}^{r} \label{eq:edge_max}
\end{align}
\end{subequations}
which maximizes how many additional nodes the agent visits while satisfying capacity constraints before returning to the depot to re-charge. Note that the capacity constraints are enforced \textit{robustly} with respect to the estimated bound set $\theta \in \hat{\Theta}^{r}$, since the true value of $\theta$ is not known.
The objective function maximizes how many total nodes are visited by the agents during iteration $r$. 
The variable ${c} \in \mathbb{R}^{|V| \times |V| \times L}$ tracks the capacity states along the routes, with $c_{i,j,l} \in [0, C_l]$ tracking the value of $c_l$ as the agent departs node $v_i$ towards node $v_j$.
Constraints (\ref{eq:leaving}) - (\ref{eq:returntodepot}) are as in (\ref{eq:mainproblem}), ensuring that each node is visited by at most one agent, that paths are well-defined, and that a single route beginning at node $n^{r}_k$ and ending at the depot is created.
Constraints (\ref{eq:edge_initial}) - (\ref{eq:edge_max}) are defined for each capacity state, enforcing the capacity state constraints and dynamics according to (\ref{eq:hldynamics}). 

\subsection{High-Level Dynamics as LTV System}
Note that the high-level dynamics (\ref{eq:hldynamics}) may be compactly written as a linear time-varying system
\begin{align*}
    n^r_{k+1} &= B_n u^r_k \\
    c^r_{k+1} &= c^r_k + B_c(\hat{\theta}^{r}) u^r_k
\end{align*}
where 
\begin{align*}
    & u^r_k  \in \mathbb{B}^{|V|^2},~ ||u^r_k||_1 = 1 \\
    & B_n = {\bf{1}}_{1,|V|} \otimes [1, 2, \dots, |V|]\\
    &B_c(\hat{\theta}^r) \in \mathbb{R}^{L \times V^2},~ B_c(\hat{\theta}^r)_{i,j} =  \hat{\theta}^r_{a, b, i}\\
    & a = [{\bf{1}}_{1,|V|}, {\bf{2}}_{1,|V|}, \dots, {\bf{V}}_{1,|V|}]_j,~~ b = B_{n,j}
\end{align*}

\section*{ACKNOWLEDGMENT}
Part of this research was carried out at the Jet Propulsion Laboratory, California Institute of Technology, under a contract with the National Aeronautics and Space Administration (80NM0018D0004).

\bibliographystyle{IEEEtran}
\bibliography{IEEEfull,bib}

\end{document}